%% file: main.tex
\documentclass{article}
\usepackage[utf8]{inputenc}

\input{packages.tex}

\input{definitions.tex}
\input{settings.tex}

\bibliography{abbrevs.bib,BL.bib,IP.bib}

\title{Two-layer neural networks with values in a Banach space}
\author{Yury Korolev\thanks{Department of Mathematical Sciences, University of Bath, Claverton Down BA2 7AY, UK. email: \texttt{ymk30@bath.ac.uk}}}
\date{}

\begin{document}

\maketitle

%%%%%%%%%%%%%%%%%%%%%%%%%%%%%%%%%%%%%%%%%%%%%%
\begin{abstract}
We study two-layer neural networks whose domain and range are Banach spaces with separable preduals. In addition, we assume that the image space is equipped with a partial order, i.e. it is a Riesz space. As the nonlinearity we choose the lattice operation of taking the positive part; in case of $\R^d$-valued neural networks this corresponds to the ReLU activation function. We prove inverse and direct approximation theorems with Monte-Carlo rates \revone{for a certain class of functions}, extending existing results for the finite-dimensional case. In the second part of the paper we study, from the regularisation theory viewpoint,  \revone{the problem of finding optimal representations of such functions via signed measures on a latent space} from a finite number of noisy observations. We discuss regularity conditions known as source conditions and obtain convergence rates in a Bregman distance \revone{for the representing measure} in the regime when both the noise level goes to zero and the number of samples goes to infinity at appropriate rates.
\end{abstract}


\ 


% \listoftodos

\tableofcontents

%%%%%%%%%%%%%%%%%%%%%%%%%%%%%%%%%%%%%%%%%%%%%%
\section{Introduction}

Two-layer neural networks (also referred to as neural networks with one hidden layer) are functions $f \colon \R^d \to \R$ of the following form
\begin{equation}\label{eq:two-layer-NN}
    f(x) = \sum_{i=1}^n a_i \sigma(\sp{x,b_i}+c_i), \quad x \in \R^d,
\end{equation}
where $\{b_i\}_{i=1}^n \subset \R^d$ \revtwo{and $\{a_i\}_{i=1}^n \subset \R$} are referred to as weights, $\{c_i\}_{i=1}^n \subset \R$ as biases, and $\sigma \colon \R \to \R$ as the activation function. The brackets $\sp{\cdot,\cdot}$ denote the scalar product in $\R^d$. 
Individual summands  $\{\sigma(\sp{x,b_i}+c_i)\}_{i=1}^n$ are referred to as neurons and collectively they are referred to as the hidden layer of the network. 

The famous universal approximation theorem~\cite{cybenko:1989,hornik:1989,leshno:1993} states that if $\sigma$ is not a polynomial then any continuous function on a compact set in $\R^d$ can be approximated arbitrary well (in the supremum norm) with functions of the type~\eqref{eq:two-layer-NN}. However,  quantitative estimates (such as the number of neurons required to achieve a certain accuracy) that can be obtained in general depend on the dimension of the input space $d$, with the approximation error scaling as $\bigO(n^{-d})$. For high-dimensional inputs ($d \gg 1$) this is not satisfactory, an effect which is known as the curse of dimensionality. 
Barron~\cite{barron:1993} showed that for $L^1$ functions whose Fourier transform satisfies a certain integrability condition dimension-independent Monte-Carlo rates $\bigO(1/\sqrt{n})$ can be obtained for the approximation error in $L^2$. 

The condition introduced by Barron is,  however, too conservative, and fails for many functions that can be approximated by~\eqref{eq:two-layer-NN} with dimension-independent rates. \revone{More general spaces, termed \emph{variation norm spaces}, were studied in, e.g.,~\cite{devore:1998,kurkova:2001,kurkova:2002,barron:2008}. Roughly speaking, these spaces consist of functions whose expansion coefficients in a certain basis or frame are absolutely summable. This approach was further extended in~\cite{bach:2017} who replaced the basis/frame expansion with an integral over a compact latent space against a finite Radon measure. This work was continued in~\cite{e2019barron}, where it was shown that 
these spaces are in some sense optimal: they contain all weak-* limits of~\eqref{eq:two-layer-NN} as $n \to \infty$ if~\eqref{eq:two-layer-NN}.}  We postpone the details to Section~\ref{sec:Barron-finite-dim}. These spaces were called variation norm (or $\Barr$) spaces in~\cite{bach:2017} and Barron spaces in~\cite{e2019barron,e2020barron-representation}, not to be confused with the spaces introduced by Barron in~\cite{barron:1993}. We will use the notation $\Barr$ for these spaces. 

Similar results can be arrived at from the mean-field perspective~\cite{rotskoff:2018-NeurIPS, mei:2018, chizat:2018-NeurIPS, sirignano:2020}, where a $\frac1n$ scaling is assumed in~\eqref{eq:two-layer-NN}
\begin{equation}\label{eq:mean-field}
    f(x) = \frac1n \sum_{i=1}^n a_i \sigma(\sp{x,b_i}+c_i), \quad x \in \R^d,
\end{equation}
and individual neurons are interpreted as interacting particles moving in a potential determined by the loss function. Since the optimal (trained) coefficients $\{a_i\}_{i=1}^n$ depend on $n$ (cf. Remark~\ref{rem:mean-field-equiv}), the formulations~\eqref{eq:two-layer-NN} and~\eqref{eq:mean-field} are, in fact, equivalent. 

A related concept is that of random feature models~\cite{rahimi:2008-rfm} in reproducing kernel Hilbert spaces, which have the same form~\eqref{eq:two-layer-NN} but differ from $\Barr$ functions in the way the parameters $\{a_i\}_{i=1}^n$ and $\{b_i,c_i\}_{i=1}^n$ are trained. While in the case of $\Barr$ functions all parameters are trainable (which is sometimes referred to as active training), in random feature models the parameters $\{b_i,c_i\}_{i=1}^n$ are fixed (sampled from a given distribution over the latent space) and only the coefficients $\{a_i\}_{i=1}^n$ are trained (this is sometimes referred to as lazy training). Features can have a more general form than in~\eqref{eq:two-layer-NN}. Corresponding function spaces were called $\mathcal F_2$ spaces in~\cite{bach:2017}. While $\mathcal F_2$ functions are significantly easier to train, their approximation properties are inferior to those of $\Barr$ functions~\cite{bach:2017,domingo:2021}.

So far we have discussed neural network approximations of scalar-valued functions $f \colon \R^d \to \R$ with finite-dimensional inputs. 
However, neural networks are now being employed increasingly often  in inherently infinite-dimensional settings such as inverse problems~\cite{arridge_et_al_acta_numerica}, mathematical imaging~\cite{mccann:2017-review} and partial differential equations (PDEs)~\cite{kutyniok:2019-nn-pdes}, where both the input and the output spaces may be infinite-dimensional. Whilst discretising these problems and subsequently using learning algorithms in the finite-dimensional setting is possible, there are advantages in designing learning algorithms directly in infinite dimensions such as discretisation- or mesh-invariance~\cite{nelsen:2020}.

\revone{A first universal approximation theorem for neural networks acting between Banach spaces has been obtained in~\cite{chen:1995}. The authors show that any continuous nonlinear function mapping a compact set in a Banach space to a compact set in another one can be approximated arbitrary well with a pair of two-layer neural networks. More recently~\cite{DeepONet-1}, this construction has been extended to deep (i.e. with more than two layers) neural networks and was termed \emph{DeepONet}. In~\cite{DeepONet-2}, the authors interpret the DeepOnet as an encoder-decoder type of network and study its approximation properties in the case when the input space is a Hilbert space. A similar approach was taken in~\cite{bhattacharya:2020} to derive a different network architecture. Other existing approaches include Fourier Neural Operators~\cite{li:2020a} and graph kernel networks~\cite{li:2020b}.} 
There is also a large body of literature on kernel methods in Hilbert~\cite{micchelli:2005-vector-valued} or Banach spaces~\cite{zhang:2013-vector-valued}. 
Random feature models with Banach-space valued features are being used in PDEs to learn mappings between Banach or Hilbert spaces~\cite{nelsen:2020}.
We refer the reader to the review~\cite{alvarez:2012-kernels-review} for details on vector-valued kernel methods.  

\revone{Our main goal in this paper is to obtain quantitative estimates (i.e. convergence rates) for the approximation of nonlinear operators with two-layer networks  by extending} the results of~\cite{bach:2017,e2019barron,e2020barron-representation} from the finite-dimensional and scalar-valued case $f \colon \R^d \to \R$ to the vector-valued and infinite-dimensional case $f \colon \X \to \Y$, where $\X$ and $\Y$ are Banach spaces. \revone{An important difference from the approach taken in~\cite{DeepONet-2} is that our convergence rates are obtained under a regularity assumption on the underlying function $f$, while~\cite{DeepONet-2} relies on the properties of the measure (more precisely, its push-forward by $f$) on the input space $\X$ that generates the training samples and underlies the Lebesgue-Bochner space in which the approximation error is studied. }

Considering infinite-dimensional inputs is natural if one is interested in dimension-independent approximation rates, and indeed we show that standard Monte-Carlo rates can be obtained in this setting. The situation with vector-valued outputs is more difficult.
Whilst going from  functions with values in $\Y=\R$ to those with values in $\Y=\R^k$ is trivial,  a generalisation to infinite-dimensional output spaces $\Y$ is far from being so.

The paper is organised as follows. In Section~\ref{sec:linear-nonlinear} we fix our function-analytic setting and recall necessary concepts from the theory of partially ordered spaces (Riesz spaces, also known as vector lattices~\cite{Meyer-Nieberg}). We then rewrite~\eqref{eq:two-layer-NN} in the appropriate abstract form using a generalised ReLU  activation function -- the operation of taking the positive part of an element in a Riesz space, cf.~\eqref{eq:NN}. \revone{Our results can be immediately transferred to the so-called \emph{leaky ReLU} activation, which in the language of Riesz spaces is a combination of the positive and negative parts of its argument. More generally, our results apply to any positively homogeneous and weakly-* continuous activation.} 

Section~\ref{sec:Barron} contains our analysis of vector-valued $\Barr$ spaces. We recall known finite-dimensional results from~\cite{bach:2017, e2019barron, e2020barron-representation} in Section~\ref{sec:Barron-finite-dim} before proceeding to the infinite-dimensional setting in Section~\ref{sec:Barron-inf-dim}. In Section~\ref{sec:Barron-def} we introduce vector-valued $\Barr$ spaces and prove their embeddings into appropriate Lipschitz and Bochner spaces. It is important to note that, at least with current techniques, we are only able to show continuity of vector-valued $\Barr$ functions with respect to the weak-* topology in the output space $\Y$; embeddings into Lipschitz and Bochner spaces are with respect to a particular metrisation of this topology. 
In Section~\ref{sec:Barron-approximation} we prove direct and inverse approximation theorems and obtain Monte-Carlo $\bigO(1/\sqrt{n})$ approximation rates, where $n$ is the number of summands (neurons) in~\eqref{eq:two-layer-NN}. In Section~\ref{sec:examples} we give examples of output spaces $\Y$ for which these results hold. These are spaces where the non-linearity $\sigma \colon \Y \to \Y$ is weakly-* continuous. In the case of the ReLU function this turns out to be a significant restriction: among the common spaces only sequence spaces $\ell^p$, $1 < p \leq \infty$, and Lipschitz spaces seem to have this property.

In Section~\ref{sec:training} we turn our attention to training vector-valued $\Barr$ functions. It is well known (e.g.,~\cite{Burger_Engl:2000, devito:2005-ip}) that training a neural network from noisy observations is an ill-posed inverse problem and therefore regularisation is required. It may be achieved implicitly by introducing stochasticity into gradient descent dynamics and/or early stopping~\cite{chizat:2018-NeurIPS,soudry:2018-bias,Jin:2018-SGD,chizat:2020-bias} or explicitly by combining the empirical  loss with an appropriate norm penalty~\cite{bach:2017}. The choice of this norm determines the model class in the language of approximation theory~\cite{devore:2021}.

\revone{It has also been shown~\cite{parhi:2021} that training sufficiently wide two-layer networks with weight decay~\cite{krogh:1991} is equivalent to solving a regularised variational problem with an $\Barr$ norm penalty term.} This problem boils down to the reconstruction of a Radon measure over a compact metric space from a finite number of point observations (pairings with continuous functions). It is a well-studied problem. In the context of deterministic regularisation, it has been studied in~\cite{bredies2013measures}. Support localisation and sparse spikes recovery under certain separation conditions have been studied in~\cite{Duval:2015,poon:2019-support}. 
\revone{In these works, the data term is considered fixed, which in our setting means that the number of data points is fixed, too. We are interested in the regime when the number of data points goes to infinity. We will also assume in our setting that the noise in the training data goes to zero, which is common in the inverse problems literature, but different from the statistical learning setting. The reason for this is that the problem of reconstructing a measure is ill-posed and therefore convergence of statistical estimators can be arbitrary slow (e.g.,~\cite{krzyzak:2002}) unless regularisation is used. For spectral regularisation, convergence rates in the statistical setting were obtained in~\cite{blanchard:2018}. Rates for variational regularisation with convex $p$-homogeneous functionals, $p>1$, were obtained in a recent paper~\cite{bubba2021stat}. Our case of a $1$-homogeneous functional (the Radon norm of a measure) is not covered by existing literature and studying it is beyond the scope of this paper.
}

Combining standard machinery of variational regularisation such as source conditions and Bregman distances~\cite{Benning_Burger_modern:2018} with Monte-Carlo integration, we obtain convergence rates in Bregman distance \revtwo{that hold with high probability} and derive an optimal choice rule of the regularisation parameter as a function of the noise in the data and the sample size.

\paragraph{Main contributions.} Below we summarise our main contributions.
\begin{itemize}
    \item We extend the definition of $\Barr$ functions with the ReLU activation function to the vector-valued case and show well-posedness of this definition if the output space $\Y$ is a Riesz space (vector lattice) where lattice operations are weakly-* sequentially continuous;
    \item we prove embeddings of the $\Barr$ space into Lipschitz and Bochner spaces defined using an appropriate metrisation of the weak-* topology on $\Y$;
    \item we prove direct and inverse approximation theorems for $\Barr$ functions and obtain Monte-Carlo approximation rates;
    \item we analyse variational regularisation in the context of \revone{finding an optimal representation} of an $\Barr$ function via an integral over a compact latent space, provide an interpretation of the source condition in terms of the support of the measure that realises the $\Barr$ function and obtain convergence rates in Bregman distance \revtwo{(with high probability)} in the regime of vanishing noise and growing amount of data.
\end{itemize}

\paragraph{Notation.} \revone{
Throughout the paper, $\X$ will be the input space of the neural network and $\Y$ its image space. Both will be assumed to have separable preduals $\predual{\X}$ and $\predual{\Y}$, respectively, and $\Y$ will be assumed to have a partial order (i.e., $\Y$ is a Riesz space). The latent space will be denoted by $\Z$. The space of linear operators $\X \to \Y$ will be denoted by $\calL(\X;\Y)$ and the space of nuclear operators by $\calN(\X;\Y)$. The unit ball in $\Y$ will be denoted by $\Ball{\Y}$, and ditto for other spaces. A metric that metrises the weak-* topology on the unit ball in $\Y$ will be denoted by $d_*$ and $(\Y,d_*)$ will stand for the space $\Y$ equipped with this weak-* metric, ditto for other spaces when appropriate. This space is not complete, but we will only work with bounded sets in $\Y$, so lack of completeness will not cause any issues. 
}

%%%%%%%%%%%%%%%%%%%%%%%%%%%%%%%%%%%%%%%%%%%%%%
\section{Linear-nonlinear decompositions in Riesz spaces}\label{sec:linear-nonlinear}
Let $\X, \Y$ be Banach spaces. We are interested in approximating continuous functions $f \colon \X \to \Y$ by two-layer neural networks. Let $\Z$ be a \emph{latent space}, to be specified later, and let $B \colon \X \to \Z$ and $A \colon \Z \to \Y$ be linear \revtwo{bounded} operators and $c \in \Z$. 
We will consider approximations of the following form, cf.~\eqref{eq:two-layer-NN}
\begin{equation}\label{eq:NN}
    f_{A,B,c}(x) \defeq A \sigma(Bx+c),
\end{equation}
where $\sigma \colon \Z \to \Z$ is a non-linear function. This representation of neural networks is often referred to as linear-nonlinear decomposition. 
We will consider the ReLU activation function, which in finite-dimensional spaces is defined  as follows 
\begin{equation}
    ReLU(z) \defeq  (z^1_+,...,z^m_+)^T = z_+, \quad z \in \R^m,
\end{equation}
i.e. it takes the componentwise positive part of a vector $z \in \R^m$. This operation makes sense in a much more general setting of Riesz spaces~\cite{Meyer-Nieberg} (also known as vector lattices). 
\begin{definition}[Riesz space]
A vector space $E$ equipped with a partial order ``$\leq$'' is called a Riesz space if for any $x,y,z \in E$ and any scalar $\lambda \geq 0$
\begin{itemize}
    \item $x \leq y \implies x+z \leq y+z$;
    \item $x \leq y \implies \lambda x \leq  \lambda y$;
    \item a supremum $x \vee y \in E$ exists, i.e. an element such that $x \vee y \geq x$, $x \vee y \geq y$ and
    \begin{equation*}
        \forall z \text{ s.t. } z \geq x, \,\, z \geq y \implies z \geq x \vee y.
    \end{equation*}
\end{itemize}
\end{definition}

In a Riesz space, the positive and the negative parts of an element $x \in E$ and its absolute value are defined as follows
\begin{equation}\label{eq:lattice_ops}
    x_+ \defeq x \vee 0, \quad x_- \defeq (-x)_+, \quad \abs{x} \defeq x_+ + x_-.
\end{equation}
Clearly, one has
\begin{equation*}
    x = x_+ - x_-, \quad x \in E.
\end{equation*}

\begin{definition}[Banach lattice]
Let $E$ be a Riesz space equipped with a norm with the following monotonicity property
\begin{equation}\label{eq:mononotic_norm}
    \abs{x} \leq \abs{y} \implies \norm{x} \leq \norm{y} \quad \forall \,\,x,y \in E.
\end{equation}
If, in addition, $E$ is norm complete, its is called a Banach lattice.
\end{definition}

\revone{We will only come across Banach lattices in \cref{sec:examples}. In the rest of the paper, the monotonicity property~\eqref{eq:mononotic_norm} will not be needed. }
For a more detailed introduction to Riesz spaces we refer to~\cite{Meyer-Nieberg}. 
\begin{example}
Here we give some common examples of Banach lattices and an example of a Riesz space that is \emph{not} a Banach lattice.
\begin{itemize}
    \item finite-dimensional spaces $\R^m$ and sequence spaces $\ell^p$ are Banach lattices under the following componentwise ordering
    \begin{equation*}
        x \geq y \iff x_i \geq y_i \quad \forall i;
    \end{equation*}
    \item the space of continuous functions $C(\Omega)$ ($\Omega$ -- compact) is a Banach lattice under the following pointwise ordering
    \begin{equation*}
        f \geq g \iff f(t) \geq g(t) \quad \forall t \in \Omega;
    \end{equation*}
    \item the spaces $L^p(\Omega)$, $1 \leq p \leq \infty$, are Banach lattices under the following ordering
    \begin{equation*}
        f \geq g \iff f(t) \geq g(t) \quad \text{a.e. in $\Omega$};
    \end{equation*}
    \item the space of signed Radon measures $\M(\Omega)$ is a Banach lattice under the following ordering
    \begin{equation*}
        \mu \geq \nu \iff \mu(A) \geq \nu(A) \quad \forall \, A \subset \Omega.
    \end{equation*}
    The positive and the negative parts of a measure are given by the Hahn-Jordan decomposition;
    \item \revtwo{ let $X,Y$ be Banach lattices. A linear operator $A \colon X \to Y$ is called \emph{positive}, $A \geq 0 $, if for all $x \geq 0$ one has $Ax \geq 0$ (this is not to be confused with positivity in the sense of positive semidefiniteness). A linear operator that can be written as a difference of two positive operators is called \emph{regular}. 
    Regular operators are always bounded. 
    The subspace of regular operators $L^r(X; Y)$ is a Banach lattice under the following partial order}
    \begin{equation*}
        A \geq B \iff Ax \geq Bx \quad \forall x \geq 0;
    \end{equation*}
    \item the space of Lipschitz functions $\Lip(\Omega)$ with the pointwise ordering
    \begin{equation*}
        f \geq g \iff f(t) \geq g(t) \quad \forall t \in \Omega
    \end{equation*}
    is a Riesz space but \emph{not} a Banach lattice, i.e. the monotonicity property~\eqref{eq:mononotic_norm} fails~\cite{weaver:2018}. 
\end{itemize}
\end{example}

Throughout the paper, we will assume that the latent space $\Z$ is a Riesz space and will understand the ReLU function as the (nonlinear) operation of taking the positive part
\begin{equation}
ReLU(z) \defeq z_+, \quad z \in \Z.
\end{equation}
Sometimes we will keep the notation $\sigma(\cdot)$ to denote the nonlinearity in~\eqref{eq:NN}. \revone{Our results also apply verbatim to the \emph{leaky ReLU} activation}
\begin{equation}
leaky\,\,ReLU(z) \defeq z_+ - \lambda z_-, \quad \lambda \in (0,1), \quad z \in \Z.
\end{equation}

Since $\sigma$ is positively one-homogeneous,~\eqref{eq:NN} is invariant under the transformation $(A,B,c)\to(A/t,tB,tc)$, hence without loss of generality we can assume that
\begin{equation}\label{eq:normalisation}
    \norm{B}_{\X \to \Z} \leq 1, \quad \norm{c}_\Z \leq 1.
\end{equation}

\begin{remark}[Compact notation]\label{rem:notation}
We will use the following compact notation (cf., e.g.,~\cite{e2020barron-representation}). Let us augment the she space $\X$ with $\R$, i.e. consider the Cartesian product $\X \times \R$. The pair $(B,c)$ will be understood as a linear operator $\X \times \R \to \Z$ acting as $(x,\alpha) \mapsto Bx + \alpha c$. Abusing notation, we will  refer to the cartesian product $\X \times \R$ as just $\X$, to $(x,\alpha)$ as just $x$ and to $(B,c)$ as $B$. Then~\eqref{eq:NN} becomes
\begin{equation}\label{eq:NN_short}
    f_{A,B}(x) \defeq A \sigma(Bx).
\end{equation}
For inputs of the form $(x,1)$ this is the same as~\eqref{eq:NN}.
\end{remark}

\begin{remark}\label{rem:predual}
If $\Y = \R$, then the operator $A$ from~\eqref{eq:NN_short} is acting $A \colon \Z \to \R$ and hence it can be identified with an element of the dual space $a \in \Z^*$. Equation~\eqref{eq:NN_short} can then be rewritten in terms of a dual pairing
\begin{equation}\label{eq:NN_dual_pairing}
    f_{a,B}(x) \defeq \sp{\sigma(Bx),a}.
\end{equation}
\end{remark}

\begin{example}[Countably wide two-layer neural networks \cite{e2020barron-representation}]\label{ex:countably_wide} 
Let $\X = \R^d$ and $\Y=\R$. Choose $\Z=\ell^\infty$ 
and $B \colon \R^d \to \ell^\infty$ that maps $x \mapsto (b_i^T x)_{i=1}^\infty$. The requirement $\norm{B}_{\R^d\to\ell^\infty} \leq 1$ means that we need $\norm{b_i} \leq 1$ for all $i$. Hence, the $b_i$'s lie in the unit ball of the dual space of $\X$, which in this case coincides with $\X$. In this way, we obtain the space of all countably wide two-layer neural networks
\begin{equation}\label{eq:countably_wide}
      \hat \F_\infty \defeq \left\{ \sum_{i=1}^\infty a_i (\sigma(Bx))_i \colon
    \sum_{i=1}^\infty \abs{a_i} < \infty \right\}
\end{equation}
with the following norm
\begin{equation}\label{eq:norm_countable}
    \norm{f}_{ \hat \F_\infty} \defeq \inf_{a,B} \{\norm{a}_{\ell^1} \colon  f \equiv f_{a,B} \} = \inf_{a,B} \left\{\sum_{i=1}^\infty \abs{a_i} \colon f \equiv f_{a,B} \right\}.
\end{equation}
Here we implicitly restrict $a$ to be an element of the predual space $\ell^1$ rather then the dual $(\ell^\infty)^*$. 
A generalisation to the case $\Y = \R^k$ is easily obtained by taking $\Z$ to be a vector-valued sequence space $\Z=\ell^\infty(\R^k)$.
\end{example}

%%%%%%%%%%%%%%%%%
\section{Variation norm spaces}\label{sec:Barron}
 
%%%%%%%%%%%%%%%%%
\subsection{The finite-dimensional case}\label{sec:Barron-finite-dim}

Variation norm spaces $\Barr$~\cite{bach:2017} (also termed Barron spaces in~\cite{e2019barron,e2020barron-representation}) consist of functions that can be written in the following form
\begin{equation*}
    f(x) = \int_{\mathcal V} \phi_v(x) \, d\mu(v),
\end{equation*}
whee $\phi_v$ are \emph{features} parametrised by the elements of a compact topological space $\mathcal V$ and $\mu$ is a signed Radon measure on $\mathcal V$. 
If $\X = \R^d$ and $\Y = \R$, the latent space $\mathcal V$ is typically chosen as the unit sphere (or unit ball) in $\R^d$. 

For features of the form $\phi_v(x) = \sigma(\sp{v,x})$ with a positively one-homogeneous activation $\sigma$, variation norm functions can be written in terms of the linear-nonlinear decomposition as shown in~\cite{e2020barron-representation}. To this end, let us choose $\Z \defeq \C(\Ball{\R^d})$ to be the space of continuous functions on the unit ball $\Ball{\R^d}$ and let
\begin{equation}\label{eq:B-finite-dim}
    B \colon \R^d \to \C(\Ball{\R^d}), \quad x \mapsto \ell_x(\cdot), \quad \text{where $\ell_x(b) \defeq b^T x, \quad b \in \Ball{\R^d}$}.
\end{equation}
Since $\norm{b}_{\R^d} \leq 1$. the requirement $\norm{B}_{\X \to \Z} \leq 1$ is clearly met.

\begin{remark}\label{rem:Rd_dual}
If $\X = \R^d$ is equipped with a $p$-norm with $p \neq 2$, the latent space $\Z$ is defined over the unit ball in the dual space of $\X$, i.e. $\R^d$ with the dual $q$-norm ($q$ is the H\"older conjugate of $p$). Hence, the latent space can be written as the space of continuous function over the unit ball in the dual space of $\X$
\begin{equation*}
    \Z = \C(\Ball{\X^*}).
\end{equation*}
If $\X$ is infinite-dimensional, the topology on $\Ball{\X^*}$ needs to be specified. In order to preserve compactness, weak-* topology on $\Ball{\X^*}$ is the natural choice. We will come back to this in Section~\ref{sec:Barron-inf-dim}.
\end{remark}

According to Remark~\ref{rem:predual}, the operator $A$ can be identified with a Radon measure $a \in C(\Ball{\R^d})^* = \M(\Ball{\R^d})$. The expression~\eqref{eq:NN_dual_pairing} then reads as follows
\begin{equation}\label{eq:Barron-f_a}
    f_a(x) = \int_{\Ball{\R^d}} \sigma(Bx)\, da.
\end{equation}

We note the following simple fact.
\begin{prop}\label{prop:NN_Lipschitz}
A function $f_a$ as defined in~\eqref{eq:Barron-f_a} is Lipschitz with constant $\norm{a}_{\M}$, i.e. for any $x,x' \in \X$ we have
\begin{equation*}
    \abs{f_a(x) - f_a(x')} \leq \norm{a}_{\M} \norm{x-x'}_\X. 
\end{equation*}
\end{prop}
\begin{proof}
This follows from the fact that the operation $\sigma(z) = z_+$ is $1$-Lipschitz in $\C(\Ball{\R^d})$.
\end{proof}
Since $f_a(0)=0$, we conclude that $f_a \in \Lip_0(\X)$, the space of Lipschitz functions on $\X$ that vanish at zero, and
\begin{equation*}
    \norm{f_a}_{\Lip_0} \leq \norm{a}_{\M}.
\end{equation*}
(Recall that, according to Remark~\ref{rem:notation}, $\X$ is a shorthand notation for $\X \times \R$ and neural networks correspond to inputs of the form $(x,1)$ and they will not vanish at a common base point. However, this does not prevent us from viewing the function~\eqref{eq:Barron-f_a} as a Lipschitz function with base point at zero in $\X \times \R$.)

It is known that not all Lipschitz functions can be written in the form~\eqref{eq:Barron-f_a}. The space of functions that can be written in this form was termed the variation norm space in~\cite{bach:2017} and the Barron space in~\cite{e2019barron}. 
\begin{definition}[Scalar-valued {$\Barr$ function}s]\label{def:Barron-finite-dim}
Let $\X = \R^d$. The space of (scalar-valued) {$\Barr$ function}s is defined as follows
\begin{equation}
    \Barr(\X) \defeq \{f \in \Lip_0(\X) \colon \norm{f}_\Barr < \infty\},
\end{equation}
where
\begin{equation}\label{eq:Barron_norm}
    \norm{f}_\Barr \defeq \inf_{a \in \M(\Ball{\R^d})}\{\norm{a}_{\M} \colon f(x)=f_a(x), \quad x \in \X\}.
\end{equation}
\end{definition}
\begin{remark}\label{rem:Barron-L1}
Since Lipschitz functions grow at most linearly at infinity, {$\Barr$ function}s lie in the Lebesgue space $L^1_\pi(\X)$ if $\pi$ is a probability measure over $\X$ with a finite first moment, cf.~\cite{e2020barron-representation}.
\end{remark}

%%%%%%%%%%%%%%%
\subsection{The infinite-dimensional case}
\label{sec:Barron-inf-dim}

In this section we study the vector-valued case when both $X$ and $\Y$ are (in general, infinite-dimensional) Banach spaces. As noted in Remark~\ref{rem:Rd_dual}, if the input space $\X$ is infinite-dimensional, the latent space $\Z = \C(\Ball{\X^*})$ needs to be considered over the unit ball \revtwo{$\Ball{\X^*}$} equipped with the weak-* topology in order to preserve the duality with Radon measures. In the vector-valued case $\Y \neq \R$, as we shall see, the latent space $\Z$ needs to be chosen as the space of $\Y$-valued continuous functions over the unit ball in the spaces of linear bounded operators $\calL(\X;\Y)$. Understanding the dualities that are involved in this case requires results from vector-valued measures~\cite{diestel_uhl:1977} and tensor products of Banach spaces~\cite{ryan2002book}.

Our setting is as follows. We will assume that $\X$ and $\Y$ are duals of separable Banach spaces $\predual{\X}$ and $\predual{\Y}$ (more precisely, we will assume that $\predual{\X}$ and $\predual{\Y}$ possess Schauder bases). We will also assume that $\Y$ has a lattice structure, i.e. that it is also a Riesz space. Although we assume $\Y$ to be both a Banach space and a Riesz space, we do not assume that it is a Banach lattice, i.e. the monotonicity property~\eqref{eq:mononotic_norm} is not required.

%%%%%%%%%%%%%%%%%%%%%%%%%%%%%%
\subsubsection{Definitions}\label{sec:Barron-def}

Denote by $\calL(\X;\Y)$ the space of linear bounded operators $\X \to \Y$ and by $\calN(\predual{\X};\predual{\Y})$ the space of nuclear  operators $\predual{\X} \to \predual{\Y}$ defined as follows~\cite{ryan2002book}
\begin{definition}[Nuclear operators]\label{def:nuclear}
Let $E$ and $F$ be Banach spaces. An operator $N \colon E \to F$ is called nuclear if it can be written in the following form
\begin{equation}\label{eq:nuc_rep}
    Nx = \sum_{i=1}^\infty \sp{x,\phi_i} y_i
\end{equation}
for some sequences $\{\phi_i\}_{i \in \N} \subset E^*$ and $\{y_i\}_{i \in \N} \subset F$ such that $\sum_{i=1}^\infty \norm{\phi_i}\norm{y_i} < \infty$. Such representation need not be unique. The nuclear norm of $N \in \calN(E;F)$ is defined as follows
\begin{equation*}
    \norm{N}_{\calN} \defeq \inf \left\{ \sum_{i=1}^\infty \norm{\phi_i}\norm{y_i} \colon N x  = \sum_{i=1}^\infty \sp{x,\phi_i} y_i \quad \forall x \right\}.
\end{equation*}
\end{definition}

The sequences $\{\phi_i\}_{i \in \N}$ and $\{y_i\}_{i \in \N}$ can also be chosen to be normalised, in which case the nuclear representation~\eqref{eq:nuc_rep} becomes
\begin{equation}\label{eq:nuc_rep-2}
    Nx = \sum_{i=1}^\infty \lambda_i \sp{x,\phi_i} y_i
\end{equation}
with some coefficients $\{\lambda_i\}_{i \in \N} \subset \R$ such that $\sum_{i=1}^\infty \abs{\lambda_i} < \infty$.

To proceed, we will need to establish compactness properties of the unit ball in $\calL(\X;\Y)$. We will need the following result.

\begin{definition}[Approximation property]
A Banach space $E$ is said to have the approximation property if for any Banach space $F$, any operator $T \colon E \to F$, any compact set $M \subset E$ and any $\eps>0$ there exists a finite-rank operator $S \colon E \to F$ such that $\norm{Tx - Sx}_F \leq \eps$ for all $x \in M$.
\end{definition}

\begin{prop}
\revtwo{Let $\X$ and $\Y$ be duals of separable Banach spaces $\predual{\X}$ and $\predual{\Y}$ which possess Schauder bases. \revtwonew{Then  $\calL(\X;\Y)$ is isometrically isomorphic to $(\calN(\predual{\X},\predual{\Y}))^*$.} For any $K \in \calL(\X;\Y)$ and $N \in \calN(\predual{\X};\predual{\Y})$ the duality pairing can be written as follows
\begin{equation}\label{eq:nuc-lin-duality-XY}
    \sp{N,K}  = \tr(KN^*) 
    \defeq \sum_{i=1}^\infty \sp{\eta_i, K x_i},
\end{equation}
where $\{x_i,\eta_i\}_{i \in \N} \subset \X \times \predual{\Y}$ is any nuclear representation~\eqref{eq:nuc_rep} of $N \in \calN(\predual{\X};\predual{\Y})$. 
Consequently, the unit ball $\Ball{\calL(\X;\Y)}$ is weakly-* compact and metrisable. }
\end{prop}
\begin{proof}
\revtwo{Any Banach space with a Schauder basis has the approximation property~\cite[Ex. 4.4]{ryan2002book}. Hence, we can use Propositions~\ref{prop:nuclear_dual}-\ref{prop:unit-ball-op-ws-compact}  with $E=\X$ and $F=\Y$ to obtain the claim.}
\end{proof}

The weak-* topology can be metrised by the following metric~\cite[Thm. V.5.1]{DS1}. Let $\{N_i\}_{i=1}^\infty$ be a countable dense system in $\calN(\predual{\X};\predual{\Y})$, which exists since $\calN(\predual{\X};\predual{\Y})$ is separable. For any ${K},{L} \in \Ball{\calL(\X,\Y)}$, define
\begin{equation*}
    d({K},{L}) = \sum_{i=1}^\infty 2^{-i} \frac{\abs{\sp{N_i,{K}-{L}}}}{1+\abs{\sp{N_i,{K}-{L}}}}.
\end{equation*}
If $\{N_i\}_{i=1}^\infty$ are normalised, we can use the following equivalent metric
\begin{equation}\label{eq:ws-metric-op}
    d_*({K},{L}) \defeq \sum_{i=1}^\infty 2^{-i} \abs{\sp{N_i,{K}-{L}}}.
\end{equation}
We will treat $\Ball{\calL({{\X}};\Y)}$ as a compact metric space with the weak-* metric~\eqref{eq:ws-metric-op} and sometimes use the notation $(\Ball{\calL({{\X}};\Y)},d_*)$ to emphasise it. 

We make the following simple observation.
\begin{prop}\label{prop:rank-one-op}
Let $x \in {{\X}}$ and $\eta \in \predual{\Y}$ be arbitrary but fixed. Define the following rank-one operator $N_x^\eta \colon \predual{\X} \to \predual{\Y}$  
\begin{equation}
    N_x^\eta(\xi) \defeq \sp{\xi,x} \eta, \quad \xi' \in \predual{\X}; \quad \norm{N_x^\eta} = \norm{x}\norm{\eta}, \label{eq:rank-one-op} 
\end{equation}
Then for any $K \in \Ball{\calL(\X;\Y)}$  we have
\begin{equation*}
    \sp{N_x^\eta,K} = \sp{\eta,Kx}.
\end{equation*}
\end{prop}
\begin{proof}
Since the nuclear representation of $N_x^\eta$ consists of only one term~\eqref{eq:rank-one-op}, the   claim follows directly from~\eqref{eq:nuc-lin-duality-XY}. 
\end{proof}

Let $\{\xi_j\}_{j \in \N}$ and $\{\eta_i\}_{i \in \N}$ be Schauder bases in $\predual{\X}$ and $\predual{\Y}$, respectively,  such that $\norm{\xi_j} = \norm{\eta_i} = 1$ for all $i,j \in \N$. Denote the corresponding coefficient functionals by $\{x_j\}_{j \in \N}$ and $\{y_i\}_{i \in \N}$. 
We get the following density result.
\begin{prop}\label{prop:rank-one-dense} The system of operators $\{N_{x_j}^{\eta_i}\}_{i,j \in \N}$ (cf.~\eqref{eq:rank-one-op}) is dense in $\calN(\predual{\X},\predual{\Y})$.
\end{prop}

\begin{remark}
Famously, there exist separable spaces that do not posses a basis~\cite{enflo:1973}, however, most known examples are rather exotic.
\end{remark}

We will consider the following metric on $\Y$ defined analogously to~\eqref{eq:ws-metric-op}
\begin{equation}\label{eq:ws-metric-Y}
    d_*(y_1,y_2) \defeq \sum_{i=1}^\infty 2^{-i} \abs{\sp{\eta_i,y_1-y_2}}.
\end{equation}
By the Banach-Alaoglu theorem, the unit ball $\Ball{\Y}$ is compact with respect to this metric. \revtwo{We will denote by $(\Y,d_*)$ the space $\Y$ equipped with this metric.  
It can be turned into a normed space by defining
\begin{equation}\label{eq:ws-norm}
    \norm{y}_{(\Y,d_*)} \defeq d_*(y,0).
\end{equation}
It is easy to verify that this functional is indeed a norm, i.e. it is absolutely one-homogeneous, satisfies the triangle inequality and $\norm{y}_{(\Y,d_*)} = 0 \iff y=0$. The space $(\Y,d_*)$ is not complete, but we will work with strongly bounded (i.e. in the norm of $\Y$) sets, so this incompleteness will not present a problem. The main point is that $\Y$ and $(\Y,d_*)$ have different metrics. }

Since $\Y$ is a Riesz space, the norm in $(\Y,d_*)$ can be equivalently expressed as follows
\begin{equation}\label{eq:ws-norm-2}
    \norm{y}_{(\Y,d_*)} = d_*(y,0) = \sum_{i=1}^\infty 2^{-i} \abs{\sp{\eta_i,y}} = \sum_{i=1}^\infty 2^{-i} \abs{\sp{\eta_i,y_+-y_-}} = d_*(y_+,y_-).
\end{equation}
A similar construction can be used to define a norm on the completion of the space of zero-mean Radon measures with respect to the Kantorovich-Rubinstein distance~\cite{weaver:2018,Bogachev:2007}.

In the previous section we defined the space $\Z$ as the space of continuous functions on the unit ball of the dual space of $\X$, cf. Remark~\ref{rem:Rd_dual}. A natural generalisation to the vector-valued case is to take $\Z$ to be the space of $\Y$-valued functions on $\Ball{\calL({{\X}};\Y)}$ that are continuous with respect to the weak-* topologies in $\Ball{\calL({{\X}};\Y)}$ and $\Y$. Hence, we take $\Z = \C(\Ball{\calL({{\X}};\Y)};(\Y,d_*))$, where $\Ball{\calL({{\X}};\Y)}$ is considered as a compact metric space with the weak-* metric~\eqref{eq:ws-metric-op}, and equip it with the following norm
\begin{equation}\label{eq:Z-def-inf-dim}
    \Z \defeq \C(\Ball{\calL({{\X}};\Y)};(\Y,d_*)), \quad \norm{f}_{\C} \defeq \sup_{K \in \Ball{\calL(\X;\Y)}} \norm{f(K)}_{(\Y,d_*)},
\end{equation}
where $\norm{\cdot}_{(\Y,d_*)}$ is as defined in~\eqref{eq:ws-norm}.

The space $\C(\Ball{\calL({{\X}};\Y)};(\Y,d_*))$ naturally inherits the lattice structure from $\Y$. To this end, we define the following partial order on $\C(\Ball{\calL({{\X}};\Y)};(\Y,d_*))$
\begin{equation}\label{eq:po_Lip}
    \phi \geq \psi \iff \phi(K) \geq_\Y \psi(K) \quad \forall K \in \Ball{\calL({{\X}};\Y)}, \quad  \phi,\psi \in \C(\Ball{\calL({{\X}};\Y)};(\Y,d_*)),
\end{equation}
where $\leq_\Y$ denotes the partial order in $\Y$. It can be verified that $\C(\Ball{\calL({{\X}};\Y)};(\Y,d_*))$ is a Riesz space. However, it is not a Banach lattice with the supremum norm~\eqref{eq:Z-def-inf-dim}. The reason is that the $\Y^*$-norm~\eqref{eq:ws-norm} lacks the monotonicity property~\eqref{eq:mononotic_norm}.

Now we define the operator $B$ analogously to Section~\ref{sec:Barron-finite-dim}
\begin{equation}\label{eq:B-def-inf-dim}
    B \colon {{\X}} \to \C(\Ball{\calL({{\X}};\Y)};(\Y,d_*)), \quad x \mapsto \calL_x(\cdot), \quad \calL_x(K) \defeq Kx \quad \forall K \in \Ball{\calL({{\X}};\Y)}.
\end{equation}
The function $\calL_x$ maps $\Ball{\calL(\X;\Y)}$ into the ball $\Ball{\Y,\norm{x}}$ of radius $\norm{x}$. We need to check that $\calL_x$ is continuous, i.e. that the operator $B$ is well-defined.

\begin{prop}\label{prop:Lx-cont}
The function $\calL_x(\cdot)$ defined in~\eqref{eq:B-def-inf-dim} is sequentially continuous with respect to the weak-* topologies on $\Ball{\calL({{\X}};\Y)}$ and $\Y$ for any fixed $x \in {{\X}}$. Consequently, $\calL_x(\cdot) \in \C(\Ball{\calL({{\X}};\Y)};(\Y,d_*))$.
\end{prop}
\begin{proof}
Let $\{K_n\}_{n \in \N} \subset \Ball{\calL({{\X}};\Y)}$ be an arbitrary sequence such that $K_n \wsto K$ for some $K \in \Ball{\calL({{\X}};\Y)}$. 
Fix $x \in \X$ and let $\eta \in \predual{\Y}$ be arbitrary. Let $N_x^\eta$ be the rank-one operator defined in~\eqref{eq:rank-one-op}. Then we have
\begin{equation*}
    \sp{N_x^\eta,K_n} \to \sp{N_x^\eta,K} \quad \text{as $n \to \infty$}
\end{equation*}
and therefore
\begin{equation*}
    \sp{\eta,K_n x} \to \sp{\eta,Kx}  \quad \text{as $n \to \infty$}
\end{equation*}
by Proposition~\ref{prop:rank-one-op}.  But this means that
\begin{equation*}
    \sp{\eta,\calL_x(K_n)} \to \sp{\eta,\calL_x(K)}  \quad \text{as $n \to \infty$}.
\end{equation*}
Since this holds for all $\eta \in \predual{\Y}$, we have $\calL_x(K_n) \wsto \calL_x(K)$.
\end{proof}

\begin{cor}
We have that $\norm{B}_{{{\X}} \to \C(\Ball{\calL({{\X}};\Y)};(\Y,d_*))} \leq 1$. Indeed, for any $x \in \X$ with $\norm{x} \leq 1$ and any $K \in \Ball{\calL({{\X}};\Y)}$,
\begin{equation*}
    \norm{\calL_x(K)}_{\Y^*} = d_*(\calL_x(K),0) = \sum_{i=1}^\infty 2^{-i} \abs{\sp{\eta_i,Kx}} \leq \sup_{i \in \N} \abs{\sp{\eta_i,Kx}} \leq \norm{K}\norm{x} \leq 1,
\end{equation*}
which also holds for the supremum in $x \in \Ball{\X}$ and $K \in \Ball{\calL({{\X}};\Y)}$.
\end{cor}

Now we need to make sure that the nonlinearity $\sigma$ does not take us outside the space of weakly-* continuous functions, i.e. that $\sigma(\calL_x(\cdot)) \in \C(\Ball{\calL({{\X}};\Y)};(\Y,d_*))$.

\begin{prop}\label{prop:sigmaLx-cont}
Suppose that lattice operations in $\Y$ are weakly-* sequentially continuous. Then the function $\sigma(\calL_x(\cdot))$ is sequentially continuous with respect to weak-* topologies on $\Ball{\calL({{\X}};\Y)}$ and $\Y$ for any fixed $x \in \X$.
\end{prop}
\begin{proof}
We recall that for any $K \in \Ball{\calL({{\X}};\Y)}$, $\sigma(\calL_x(K))=(Kx)_+$, where the positive part is taken in $\Y$. Let $K_n \wsto K \in \Ball{\calL({{\X}};\Y)}$. Then $\calL_x(K_n) \wsto \calL_x(K)$ in $\Y$
by Proposition~\ref{prop:sigmaLx-cont}.  Since lattice operations in $\Y$ are weakly-* sequentially continuous, we also get that $(\calL_x(K_n))_+ \wsto (\calL_x(K))_+$ in $\Y$.
\end{proof}

\begin{remark}\label{rem:lattice-op-ws-cont}
It turns out that, apart from the trivial case $\dim(\Y)<\infty$, lattice operations are rarely sequentially continuous in the weak or weak-* topologies. For example, lattice operations are not weakly-* sequentially continuous in all $L_p([0,1])$ spaces for $1 < p \leq \infty$ (consider the sequence $f_n(x) \defeq \sin(2\pi n x) \wsto 0$, while $\sp{(f_n)_+,\one} = \int_0^1 (\sin(2\pi n x))_+ dx = \sfrac{1}{\pi}$ for all $n$). However, lattice operations are weakly-* sequentially continuous in sequence spaces $\ell^p$, $1 < p \leq \infty$, and in spaces of Lipschitz functions (see Section~\ref{sec:examples} for details). 
\end{remark} 

For readers' convenience, we provide here a proof of sequential weak-* continuity of lattice operations in $\ell^p$ spaces. In fact, lattice operations can even be $1$-Lipschitz with respect to the $d_*$ metric~\eqref{eq:ws-metric-Y}.
\begin{prop}\label{prop:ell-p-lattice-cont}
Let $\Y = \ell^p$, $1<p\leq\infty$, and let the $d_*$ metric~\eqref{eq:ws-metric-Y} be defined using the standard basis $\{e_i\}_{i=1}^\infty$ as the countable dense system in the predual. Then the lattice operations are $1$-Lipschitz  with respect to the $d_*$ metric.
\end{prop}
\begin{proof}
Let $\eta_n \wsto \eta$ on $\Ball{\ell^p}$. Taking the standard basis $\{e_i\}_{i=1}^\infty$ as the countable dense system in the predual $\ell^q$ ($q$ is the conjugate exponent of $p$), we get from~\eqref{eq:ws-metric-Y}
\begin{equation*}
    d_*((\eta_n)_+,\eta_+) = \sum_{i=1}^\infty 2^{-i} \abs{(\eta^i_n)_+ - \eta^i_+} \leq \sum_{i=1}^\infty 2^{-i} \abs{\eta^i_n - \eta^i} = d_*(\eta_n,\eta),
\end{equation*}
where $\eta^i_n$ and $\eta^i$ denote the $i$-th coordinate of $\eta_n$ and $\eta$, respectively. The inequality above holds because lattice operations are $1$-Lipschitz in $\R$.
\end{proof}

We will identify the operator $A$ with a measure $a \in \M(\Ball{\calL({{\X}};\Y)})$. For any $f \in \C(\Ball{\calL({{\X}};\Y)};(\Y,d_*))$ and $a \in \M(\Ball{\calL({{\X}};\Y)})$, we will denote by $\sp{f,a}$ the $\Y$-valued pairing of $f$ and $a$. More precisely, for any $K \in \Ball{\calL(\X;\Y)}$ and any $\eta_i$ from the countable system $\{\eta_i\}_{i \in \N}$ (cf. Proposition~\ref{prop:rank-one-dense}), the function
\begin{equation*}
    K \mapsto \sp{\eta_i,f(K)} 
\end{equation*}
is in $\C(\Ball{\calL({{\X}};\Y)};\R)$ and therefore the pairing
\begin{equation*}
    \sp{\sp{\eta_i,f(\cdot)},a} = \int_{\Ball{\calL}} \sp{\eta_i,f(K)} \, da(K)
\end{equation*}
is well-defined for any $a \in \M(\Ball{\calL({{\X}};\Y)})$. In the above formula, $\Ball{\calL}$ stands for  $\Ball{\calL({{\X}};\Y)}$. Using the linearity of dual pairings, we rewrite this as follows
\begin{equation*}
    \sp{\sp{\eta_i,f(\cdot)},a} = \sp{\eta_i,\int_{\Ball{\calL}} f(K) \, da(K)} = \sp{\eta_i,\sp{f(\cdot),a}}.
\end{equation*}
We will hence write  
\begin{equation}\label{eq:Y-valued-integral}
    \sp{f,a} \defeq \int_{\Ball{\calL}} f(K) \, da(K), \quad \sp{f,a} \in (\Y,d_*),
\end{equation}
understanding this as a weak-* limit. In fact, $f$ can be understood as the density of some $(\Y,d_*)$-valued measure~\cite{diestel_uhl:1977} on \revtwo{$\Ball{\calL(\X;\Y)}$} and $\sp{f,a}$ as its pairing with the constant $\one$ function. With this notation, we can generalise~\eqref{eq:NN_dual_pairing} to the vector-valued case
\begin{equation}\label{eq:Barron-f_a-inf-dim}
    f_a(x) \defeq \sp{\sigma(Bx),a}, \quad f_a \colon \X \to (\Y,d_*).
\end{equation}

 We have the following analogue of Proposition~\ref{prop:NN_Lipschitz}.

\begin{theorem}\label{thm:NN-ws-Lipschitz-inf-dim}
Let $a \in \M(\Ball{\calL(\X;\Y)})$ and suppose that lattice operations in $\Y$ are $1$-Lipschitz with respect to the $d_*$ metric. Then the function $f_a$ defined in~\eqref{eq:Barron-f_a-inf-dim} is Lipschitz with constant $\norm{a}_{\M}$ with respect to the $d_*$ metric in $\Y$, i.e. for any $x,x' \in \X$ we have
\begin{equation*}
    d_*(f_a(x),f_a(x')) \leq \norm{a}_{\M} \norm{x-x'}. 
\end{equation*}
Consequently, $f_a \in \Lip_0(\X;(\Y,d_*))$, where $\Lip_0(\X;(\Y,d_*))$ is the space of Lipschitz functions with respect to the $d_*$ metric in $\Y$ vanishing at zero. 
\end{theorem}
\begin{proof}
Let $x,x' \in \X$ be fixed and let $\{\eta_i\}_{i=1}^\infty$ be as defined in Proposition~\ref{prop:rank-one-dense}.  The $d_*$ distance between $f_a(x)$ and $f_a(x')$ is given by
\begin{eqnarray*}
d_*(f_a(x),f_a(x')) &=& \sum_{i=1}^\infty 2^{-i} \abs{\sp{\eta_i,\sp{\sigma(Bx)-\sigma(Bx'),a}}}.
\end{eqnarray*}
Using the linearity of dual pairings and remembering that $\sigma(Bx)(K) = (Kx)_+$, we get
\begin{eqnarray*}
\abs{\sp{\eta_i,\sp{\sigma(Bx)-\sigma(Bx'),a}}} &=&  \abs{\sp{\eta_i, \int_{\Ball{\calL}} ((Kx)_+ - (Kx')_+) \,da(K)}} \nonumber  \\ 
&=& \abs{\int_{\Ball{\calL}} \sp{\eta_i,(Kx)_+ - (Kx')_+} \,da(K)} \nonumber \\
&\leq& \int_{\Ball{\calL}} \abs{\sp{\eta_i, (Kx)_+ - (Kx')_+}} \,d\abs{a}(K),
\end{eqnarray*}
where  $\abs{a} \in \M(\Ball{\calL(\X;\Y)})$ is the total variation of $a$. 
Now, using the above estimate as well as Fubini's theorem, we get that
\begin{eqnarray*}
d_*(f_a(x),f_a(x')) &\leq& \sum_{i=1}^\infty 2^{-i} \int_{\Ball{\calL}} \abs{\sp{\eta_i,(Kx)_+ - (Kx')_+}} \,d\abs{a}(K) \nonumber \\
&=& \int_{\Ball{\calL}} \sum_{i=1}^\infty 2^{-i} \abs{\sp{\eta_i,(Kx)_+ - (Kx')_+}} \,d\abs{a}(K) \nonumber \\
&=& \int_{\Ball{\calL}} d_*((Kx)_+,(Kx')_+) \,d\abs{a}(K) \nonumber \\
&\leq& \int_{\Ball{\calL}} d_*(Kx,Kx') \,d\abs{a}(K), 
\end{eqnarray*}
where we used the fact that lattice operations in $\Y$ are $1$-Lipschitz with respect to the $d_*$ metric. Now, since $d_*(y,y') \leq \norm{y-y'}$ for all $y,y' \in \Y$, we have
\begin{eqnarray*}
    d_*(f_a(x),f_a(x')) &\leq& \int_{\Ball{\calL}} \norm{K(x-x')} \,d\abs{a}(K) \\
&\leq& \norm{x-x'} \int_{\Ball{\calL}} \,d\abs{a}(K) \\
&=& \norm{a}_{\M} \norm{x-x'}
\end{eqnarray*}
as desired. Noting that $f_a(0)=0$ completes the proof.
\end{proof}

We also have the following continuity result with respect to the weak-* topology on $\X$.

\begin{theorem}\label{thm:NN-ws-ws-cont}
Let $a \in \M(\Ball{\calL(\X;\Y)})$ and suppose that lattice operations in $\Y$ are weakly-* sequentially continuous. Then the function $f_a$ defined in~\eqref{eq:Barron-f_a-inf-dim} is sequentially continuous with respect to the weak-* topologies in $\X$ and $\Y$, i.e. for any sequence $x_n \wsto x$ in $\X$ we have
\begin{equation*}
    f_a(x_n) \wsto f_a(x) \quad \text{in $\Y$}. 
\end{equation*}
\end{theorem}
\begin{proof}
Let $x_n \wsto x$ and let $K \in \Ball{\calL(\X;\Y)}$ be weak-*--weak-* continuous (that is, $K$ has to be the adjoint of a bounded operator $\predual{\Y} \to \predual{\X}$). Then we have
\begin{equation*}
    Kx_n \wsto Kx
\end{equation*}
and, since lattice operations in $\Y$ are weakly-* sequentially continuous, 
\begin{equation}\label{eq:sigmaKxn-ws-cont}
    (Kx_n)_+ \wsto (Kx)_+.
\end{equation}
In particular, this holds for finite-rank operators $\X \to \Y$. Now let $\eta \in \predual{\Y}$ be arbitrary and consider
\begin{equation*}
    \sp{\eta,f_a(x_n)} =  \int_{\Ball{\calL}} \sp{\eta,(Kx_n)_+} \, da(K).
\end{equation*}
\revtwo{Since $\predual{\X}$ has the approximation property, by \cite[Prop. 4.6(iii)]{ryan2002book} finite-rank operators are weakly-* dense in $\calL(\X,\Y))$} and we can approximate the above integral with an integral over finite-rank operators, for which~\eqref{eq:sigmaKxn-ws-cont} holds. Therefore, we have
\begin{eqnarray*}
    \lim_{n \to \infty} \sp{\eta,f_a(x_n)} &=& \lim_{n \to \infty} \int_{\Ball{\calL}} \sp{\eta,(Kx_n)_+} \, da(K) = \int_{\Ball{\calL}} \lim_{n \to \infty} \sp{\eta,(Kx_n)_+} \, da(K) \\
    &=& \int_{\Ball{\calL}} \sp{\eta,(Kx)_+} \, da(K) = \sp{\eta,f_a(x)}
\end{eqnarray*}
by the dominated convergence theorem. Since $\eta \in \predual{\Y}$ was arbitrary, this implies the claim.
\end{proof}

Now we are ready to extend Definition~\ref{def:Barron-finite-dim} to the infinite-dimensional case.
\begin{definition}[Vector-valued {$\Barr$ function}s]
Let $\X,\Y$ be duals of separable Banach spaces $\predual{\X}$ and $\predual{\Y}$ that admit Schauder bases. Let $\Y$ be such that lattice operations are $1$-Lipschitz with respect to the $d_*$ metric. Define the space of $\Y$-valued {$\Barr$ function}s as follows
\begin{equation}
    \Barr(\X;\Y) \defeq \{f \in \Lip_0(\X;(\Y,d_*)) \colon \norm{f}_\Barr < \infty\},
\end{equation}
where $\Lip_0(\X;(\Y,d_*))$ is the space of Lipschitz functions with respect to the $d_*$ metric in $\Y$ that vanish at zero (cf.~Proposition~\ref{thm:NN-ws-Lipschitz-inf-dim}) and
\begin{equation}\label{eq:Barron_norm_op}
    \norm{f}_\Barr \defeq \inf_{a \in \M(\Ball{\calL(\X;\Y)})}\{\norm{a}_{\M} \colon f_a(x) = f(x) \,\, \forall x \in \X\}.
\end{equation}
We set $\norm{f}_{\Barr} = \infty$ if there is no feasible point in~\eqref{eq:Barron_norm_op}.
\end{definition}
\begin{remark}\label{rem:Barron-Bochner}
\revtwo{We emphasise that the target space in this definition is $(\Y,d_*)$ and not $\Y$.} That is, {$\Barr$ function}s are Lipschitz only with respect to the weak-* metric~\eqref{eq:ws-metric-Y} in $\Y$, cf. Theorem~\ref{thm:NN-ws-Lipschitz-inf-dim}.
\end{remark}

\begin{remark}\label{rem:Barr-in-Lp}
Let $\pi \in \P_p(\X)$ be a probability measure over $\X$ with $p$ finite moments. Then the variation norm space $\Barr(\X;\Y)$ is a subspace of the Bochner space $L^p_\pi(\X;(\Y,d_*))$. Indeed, since Lipschitz functions grow at most linearly at infinity, we get from Theorem~\ref{thm:NN-ws-Lipschitz-inf-dim} that for any $f = \sp{\sigma(B \cdot),a} \in \Barr(\X;\Y)$
\begin{eqnarray*}
    \norm{f}^p_{L^p_\pi} &=& \int_{\X} \norm{f}^p_{(\Y,d_*)} \, d\pi(x) \leq \norm{a}_{\M}^p \int_{\X} \norm{x}^p_{\X} \, d\pi(x) = m_p(\pi) \norm{f}_{\Barr}^p < \infty,
\end{eqnarray*}
where $m_p(\pi)$ is the $p$-th moment of $\pi$.
\end{remark}

The next result shows that the infimum in~\eqref{eq:Barron_norm_op} is indeed attained.
\begin{prop}\label{prop:Barr-inf-exists}
If~\eqref{eq:Barron_norm_op} is feasible, it admits a minimiser.
\end{prop}
\begin{proof}
Let $\{a_k\}_{k=1}^\infty$ be a minimising sequence. Since it is bounded, it contains a weakly-* convergent subsequence (that we don't relabel)
\begin{equation*}
    a_k \wsto a \quad \text{in $\M(\Ball{\calL(\X;\Y)})$}.
\end{equation*}
Let $\{\eta_i\}_{i=1}^\infty \subset \predual{\Y}$ be as defined in Proposition~\ref{prop:rank-one-dense}. Fix an arbitrary $x \in \X$. Since $a_k$ is feasible, we have that $f_{a_k}(x) = f(x)$ for all $x \in \X$ and for all $k \in \N$. With computations similar to those in Proposition~\ref{thm:NN-ws-Lipschitz-inf-dim}, we get that for any $i \in \N$ 
\begin{eqnarray*}
    \sp{\eta_i,f(x)} = \sp{\eta_i,f_{a_k}(x)} = \sp{\eta_i,\sp{\sigma(Bx),a_k}} &=& \int\limits_{\Ball{\calL(\X;\Y)}} \sp{\eta_i, (K x)_+}\, da_k(K) \\
    &\to& \int\limits_{\Ball{\calL(\X;\Y)}} \sp{\eta_i, (K x)_+}\, da(K) = \sp{\eta_i, f_a(x)},
\end{eqnarray*}
since the function $K \mapsto \sp{\eta_i, (K x)_+}$ is continuous on $\Ball{\calL(\X;\Y)}$ by Proposition~\ref{prop:sigmaLx-cont}. Hence, $\sp{\eta_i, f_a(x)} = \sp{\eta_i, f(x)}$ holds for all $i$ and therefore $\sp{\sigma(Bx),a}= f(x)$. This holds for all $x$, hence the feasible set in~\eqref{eq:Barron_norm_op} is weakly-* closed. Since the Radon norm is weakly-* \lsc{}, we get the claim by the direct method of calculus of variations. 
\end{proof}

%%%%%%%%%%%%%%%%%%%%%%%%%%%%%%
\subsubsection{Approximation theorems}\label{sec:Barron-approximation}

Our goal in this section is to show that $\Y$-valued {$\Barr$ function}s can be well approximated with countably wide neural networks and that, in some sense, they are the largest class of functions with this property. Similar results in the scalar-valued ($\Y=\R$) and finite-dimensional ($\X = \R^d$) case can be found in~\cite{e2019barron,e2020barron-representation}.

\begin{theorem}[Inverse approximation]\label{thm:inverse_approximation}
Let $\X,\Y$ be duals of separable Banach spaces $\predual{\X}$ and $\predual{\Y}$ that possess Schauder bases and let $\Y$ be such that lattice operations are $1$-Lipschitz with respect to the $d_*$ metric~\eqref{eq:ws-metric-Y}. Let
\begin{equation*}
    f_n(x) \defeq \sum_{i=1}^n \alpha_i (K_i x)_+, \quad x \in {{\X}},
\end{equation*}
where $K_i$ are finite-rank operators ${{\X}} \to \Y$ satisfying $\norm{K_i} \leq 1$ and $\alpha_i$ rational numbers (both $K_i$ and $\alpha_i$ depend on $n$), and suppose that
\begin{equation*}
    \sum_{i=1}^n \abs{\alpha_i} \leq C \quad \text{uniformly in $n$}.
\end{equation*}
Then  there exists $f \in \Barr({{\X}};\Y)$ with $\norm{f}_{\Barr} \leq C$ such that
\begin{enumerate}[(i)]
    \item for any $x \in {{\X}}$
    \begin{equation*}
    \sum_{i=1}^n \alpha_i (K_i x)_+ \wsto f(x) \quad \text{ in $\Y$ as $n \to \infty$,}
    %\quad \forall x \in {{\X}}.
    \end{equation*}
    \revtwo{the convergence being along a subsequence that does not depend on $x$;}
    \item if $\mu \in \P_p(\X)$ is a probability measure on $\X$ with $p \geq 1$ finite moments then convergence in the Bochner space $L^p_\mu(\X;(\Y,d_*))$ also holds \revtwo{along the same subsequence}
    \begin{equation*}
    \norm{f(x) - \sum_{i=1}^n \alpha_i (K_i x)_+}_{L^p_\mu(\X;(\Y,d_*))} \to 0  \quad \text{as $n \to \infty$.}
    \end{equation*}
\end{enumerate}
\end{theorem}
\begin{remark}\label{rem:mean-field-equiv}
Since $\alpha_i$ depend on $n$, it does not matter whether we use a normalisation $\frac1n \sum_{i=1}^n \alpha_i (K_i x)_+$ or not. In fact, this is nothing but a change of notation.
\end{remark}
\begin{proof}[Proof of Theorem~\ref{thm:inverse_approximation}]
\begin{enumerate}[\itshape(i)]
    \item 
We rewrite $f_n$ using Radon measures on the unit ball in $\calL({{\X}};\Y)$
\begin{equation*}
    f_n(x) = \sp{\sigma(Bx),\sum_{i=1}^n a_i},
\end{equation*}
where $a_i \defeq \alpha_i \delta_{K_i} \in \M(\Ball{\calL({{\X}};\Y)})$. 
Using the condition that $\sum_{i=1}^n \abs{\alpha_i} \leq C$ uniformly in $n$, with the same arguments as in Proposition~\ref{prop:Barr-inf-exists} we get that 
\begin{equation*}
    \sum_{i=1}^n a_i \wsto a \quad \text{in $\M(\Ball{\calL({{\X}};\Y)})$}
\end{equation*}
\revtwo{along a subsequence} and 
\begin{equation*}
    \norm{a}_{\M} \leq \liminf_{n \to \infty} \norm{\sum_{i=1}^n a_i}_{\M} \leq \liminf_{n \to \infty} \sum_{i=1}^n \abs{\alpha_i} \leq C.
\end{equation*}
Letting
\begin{equation*}
    f(x) \defeq \sp{\sigma(Bx),a}, \quad x \in {{\X}},
\end{equation*}
we observe that $\norm{f}_{\Barr} \leq \norm{a}_{\M} \leq C$. 
Now let $\eta \in \predual{\Y}$ be arbitrary. Then
\begin{equation*}
    \sp{\eta,f(x) - f_n(x)} = \sp{\eta, \sp{a-\sum_{i=1}^n a_i,\sigma(Bx)}} = \sp{\sp{\eta,(\, \cdot \, x)_+},a-\sum_{i=1}^n a_i} \to 0
\end{equation*}
as $n \to \infty$, since the function $K \mapsto \sp{\eta,(K x)_+}$ is weakly-* continuous on $\Ball{\calL(\X;\Y)}$. Since $\eta \in \predual{\Y}$ was arbitrary, we get that $f_n(x) \wsto f(x)$ on $\Y$.

\item 
\revtwo{
Let $\eta \in \Ball{\predual{\Y}}$ be arbitrary. Then
\begin{eqnarray*}
    \abs{\sp{\eta,f(x)-f_n(x)}} &=& \abs{\sp{\eta,\sp{\sigma(Bx),a-a_n}}} = \abs{\int_{\Ball{\calL}} \sp{\eta,(Kx)_+} \, d(a-a_n)} \\
    &\leq&  \abs{\int_{\Ball{\calL}} \abs{\sp{\eta,(Kx)_+}} \, d\abs{a-a_n}} \leq \abs{\int_{\Ball{\calL}} \norm{\eta} \norm{(Kx)_+} \, d\abs{a-a_n}} \\
    &\leq& \norm{x} \abs{\int_{\Ball{\calL}} \norm{K} \, d\abs{a-a_n}} \leq 2C\norm{x}.
\end{eqnarray*}
Therefore, for all $x\in\X$ we have
\begin{equation*}
    d_*(f(x),f_n(x)) = \sum_{i=1}^\infty 2^{-i} \abs{\sp{\eta_i,f(x)-f_n(x)}} \leq \sup_{i=1,...,n} \abs{\sp{\eta_i,f(x)-f_n(x)}}  \leq 2C\norm{x}.
\end{equation*}
Let $\mu \in \P_p(\X)$ be a probability measure on $\X$ with $p \geq 1$ finite moments. Then the function $x \mapsto \norm{x}$ is in $L^p(\X)$, hence the sequence $d_*(f(x),f_n(x))$ is dominated by an $L^p$-function. Since $d_*(f(x),f_n(x)) \to 0$ pointwise on $\X$, convergence in $L^p_\mu(\X;(\Y,d_*))$ follows from the dominated convergence theorem.
}
\end{enumerate}
\end{proof}

Next we prove approximation rates of functions in $\Barr(\X;\Y)$ with finite two-layer neural networks.

\begin{theorem}[Direct approximation]\label{thm:direct_approximation}
Let $\X,\Y$ be duals of separable Banach spaces $\predual{\X}$ and $\predual{\Y}$ that possess Schauder bases and let $\Y$ be such that lattice operations are $1$-Lipschitz with respect to the $d_*$ metric~\eqref{eq:ws-metric-Y}. Let $f \in \Barr({{\X}};\Y)$. Then for every $c > 2\sqrt{2}$ there exists a sequence of finite-rank operators $\{K_i\}_{i=1}^n \subset \calL(\X;\Y)$ with $\norm{K_i} \leq 1$ and a sequence of rational numbers $\{\alpha_i\}_{i=1}^n$ such that 
\begin{enumerate}[(i)]
    \item for any $x \in {{\X}}$
    \begin{equation*}
    d_*\left(\sum_{i=1}^n \alpha_i (K_i x)_+, f(x)\right) \leq \frac{c\norm{f}_{\Barr}\norm{x}}{\sqrt{n}};
    %\quad \forall x \in {{\X}}.
    \end{equation*}
    \item if $\mu \in \P_p(\X)$ is a probability measure on $\X$ with $p \geq 1$ finite moments then the rate also holds in the Bochner space $L^p_\mu(\X;(\Y,d_*))$ 
    \begin{equation*}
    \norm{f(x) - \sum_{i=1}^n \alpha_i (K_i x)_+}_{L^p_\mu(\X;(\Y,d_*))} \leq \frac{c\norm{f}_{\Barr}(m_p(\mu))^{\frac1p}}{\sqrt{n}},
    \end{equation*}
    where $m_p(\mu) < \infty$ is the $p$-th moment of $\mu$.
\end{enumerate}
\end{theorem}
\begin{proof}
\begin{enumerate}[\itshape(i)]
    \item 
    Since $f \in \Barr(\X,\Y)$, it can be written as follows, cf.~\eqref{eq:Barron-f_a-inf-dim}
    \begin{equation*}
        f(x) = \sp{\sigma(Bx),a} = \int_{\Ball{\calL}} (Kx)_+ \, da(K),
    \end{equation*}
    where $a \in \M(\Ball{\calL(\X;\Y)})$ is as in Proposition~\ref{prop:Barr-inf-exists}. Using the Hahn-Jordan decomposition $a = a_+ - a_-$ and normalising, we get that 
    \begin{equation}\label{eq:est-1}
        f(x)  = \norm{a_+}_{\M} \sp{\sigma(Bx),\rho^+} - \norm{a_-}_{\M} \sp{\sigma(Bx),\rho^-} \defeq f^+(x) - f^-(x),
    \end{equation}
    where $\rho^\pm \defeq \frac{a_\pm}{\norm{a_\pm}_{\M}}$  are probability measures on $\Ball{\calL(\X;\Y)}$. 
    Let $K^\pm_i \sim \rho^\pm$ be random i.i.d. samples from $\rho^\pm$, $i = 1,...,n$, and let \revtwonew{$\rho^\pm_i \defeq \frac1n \delta_{K^\pm_i}$}. Define
    \begin{equation*}
        f_n(x) \defeq \norm{a_+}_{\M}\sp{\sigma(Bx),\sum_{i=1}^n \rho^+_i} - \norm{a_-}_{\M}\sp{\sigma(Bx),\sum_{i=1}^n \rho^-_i} \defeq f^+_n(x) - f^-_n(x).
    \end{equation*}
    \revtwo{Since finite-rank operators are weakly-* dense in $\Ball{\calL(\X;\Y)}$, we can assume without loss of generality that $K^\pm_i$ are finite-rank.}

    Let $\{\eta_j\}_{j \in \N}$ be a normalised dense system in $\predual{\Y}$ as defined in Proposition~\ref{prop:rank-one-dense}. Fix $x \in \X$ and consider the following functions $\phi^x_j \in \C(\Ball{\calL(\X;\Y)})$
    \begin{equation*}
        \phi^x_j(K) \defeq \sp{\eta_j,(Kx)_+}, \quad j \in \N,
    \end{equation*}
    where continuity is with respect to the weak-* topology on $\Ball{\calL(\X;\Y)}$. With this notation, we have that $\sp{\eta_j,\sp{\sigma(Bx),\rho^\pm}} = \sp{\phi^x_j,\rho^\pm}$ and hence 
    \begin{equation*}
        d_*(f^\pm(x),f^\pm_n(x)) = 
        \norm{a_\pm}_{\M} \sum_{j=1}^\infty 2^{-j} \abs{\sp{\phi^x_j,\rho^\pm-\sum_{i=1}^n \rho^\pm_i}} = \norm{a_\pm}_{\M}  \mean_j \abs{\sp{\phi^x_j,\rho^\pm-\sum_{i=1}^n \rho^\pm_i}},
    \end{equation*}
    where we denoted by $\mean_j$ the integration over $\N$ with respect to the discrete measure $j \mapsto 2^{-j}$.
    
    Since $\rho^\pm_i$ are random, $d_*^2(f^\pm(x),f^\pm_n(x))$ is a random variable. Taking its mean over the random samples $\rho^\pm_i$ and using Jensen's inequality, we get the following estimate
    \begin{eqnarray*}
        \frac{1}{\norm{a_\pm}_{\M}^2} \mean_{\rho^\pm} \left( d_*^2(f^\pm(x),f^\pm_n(x)) \right) &=& \mean_{\rho^\pm} \left(\mean_j \abs{\sp{\phi^x_j,\rho^\pm-\sum_{i=1}^n \rho^\pm_i}}\right)^2 \\
        &\leq& \mean_{\rho^\pm} \mean_j \abs{\sp{\phi^x_j,\rho^\pm-\sum_{i=1}^n \rho^\pm_i}}^2 \\
        &=& \mean_j \mean_{\rho^\pm} \abs{\sp{\phi^x_j,\rho^\pm-\sum_{i=1}^n \rho^\pm_i}}^2.
    \end{eqnarray*}
    The central limit theorem implies that
    \begin{equation*}
        \mean_{\rho^\pm} \abs{\sp{\phi^x_j,\rho^\pm-\sum_{i=1}^n \rho^\pm_i}}^2 \leq \frac{var_{\rho^\pm}(\phi^x_j)}{n},
    \end{equation*}
    where $var_{\rho^\pm}(\phi^x_j)$ denotes the variance of $\phi^x_j$ with respect to $\rho^\pm$. This variance can be estimated as follows
    \begin{eqnarray*}
        var_{\rho^\pm}(\phi^x_j) &\defeq& \int_{\Ball{\calL}} (\phi^x_j(K))^2 \, d\rho^\pm(K) - \left( \int_{\Ball{\calL}} \phi^x_j(K) \, d\rho^\pm(K) \right)^2 \\
        &=& \int_{\Ball{\calL}} \abs{\sp{\eta_j,(Kx)_+}}^2 \, d\rho^\pm(K) - \left( \int_{\Ball{\calL}} \sp{\eta_j,(Kx)_+} \, d\rho^\pm(K) 
        \right)^2 \\
        &\leq& \int_{\Ball{\calL}} \abs{\sp{\eta_j,(Kx)_+}}^2 \, d\rho^\pm(K) \\
        &\leq& \int_{\Ball{\calL}} \abs{\sp{\eta_j,Kx}}^2 \, d\rho^\pm(K) \\
        &\leq& \int_{\Ball{\calL}} \norm{K}^2 \norm{x}^2 \, d\rho^\pm(K) \\
        &\leq& \norm{x}^2,
    \end{eqnarray*}
    since $\rho^\pm$ is a probability measure. This estimate is uniform in $j$, hence we get that
    \begin{equation*}
        \mean_{\rho^\pm} \left( d_*^2(f^\pm(x),f^\pm_n(x)) \right) \leq \frac{\norm{x}^2\norm{a_\pm}_{\M} ^2}{n}.
    \end{equation*}
    Jensen's inequality then implies that
    \begin{equation*}
        \mean_{\rho^\pm} \left( d_* \right) = \sqrt{\left(\mean_{\rho^\pm} \left( d_* \right)\right)^2} \leq \sqrt{\mean_{\rho^\pm} \left( d_*^2 \right)} \leq \frac{\norm{x}\norm{a_\pm}_{\M} }{\sqrt{n}},
    \end{equation*}
    where $d_*$ stands for $d_*(f^\pm(x),f^\pm_n(x))$. The variance of $d_*$ can be estimated similarly:
    \begin{equation*}
        var_{\rho^\pm} \left( d_* \right) = \mean_{\rho^\pm} \left(d_*^2\right) - \left(\mean_{\rho^\pm}  \left(d_* \right) \right)^2 \leq \mean_{\rho^\pm} \left(d_*^2\right) \leq \frac{\norm{x}^2\norm{a_\pm}_{\M} ^2}{n}.
    \end{equation*}
    Chebyshev's inequality implies that for any $\eps >1$
    \begin{eqnarray*}
        P\left(\abs{d_* - \mean_{\rho_\pm} \left(d_*\right)} \geq \eps \sqrt{var_{\rho^\pm}\left(d_*\right)}\right) \leq \frac{1}{\eps^2}
    \end{eqnarray*}
and therefore
\begin{eqnarray*}
    P\left(\abs{d_* - \mean_{\rho_\pm} \left(d_*\right)} < \eps \sqrt{var_{\rho^\pm}\left(d_*\right)} \right) \geq 1 - \frac{1}{\eps^2} > 0
\end{eqnarray*}
for any $\eps >1$. (We note that the empirical measures $\rho_i^+$ and $\rho_i^-$ arise from samples two different probability distributions $\rho^\pm$, which even have disjoint supports by the properties of the Hahn-Jordan decomposition.) Since the above probability is strictly positive, there exist samples $\hat K^\pm_i$ and empirical measures $\hat \rho^\pm_i$ such that the following estimate holds for $\hat f^\pm_n \defeq \norm{a_\pm}_{\M} \sp{\sigma(Bx),\sum_{i=1}^n \hat \rho^\pm_i}$ 
\begin{equation*}
    \abs{d_*(f^\pm(x),\hat f^\pm_n(x)) - \mean_{\rho_\pm} \left(d_*(f^\pm(x), f^\pm_n(x))\right)} < \eps \sqrt{var\left(d_*(f^\pm(x), f^\pm_n(x))\right)}
\end{equation*}
and 
\begin{eqnarray*}
    d_*(f^\pm(x),\hat f^\pm_n(x)) &<& \mean_{\rho_\pm} \left(d_*(f^\pm(x), f^\pm_n(x))\right) + \eps \sqrt{var\left(d_*(f^\pm(x), f^\pm_n(x))\right)} \\
    &\leq& (1+\eps)\frac{\norm{x}\norm{a_\pm}_{\M} }{\sqrt{n}}.
\end{eqnarray*}
Denoting
\begin{equation*}
    \hat f_n \defeq \norm{a_+}_{\M} \sp{\sigma(Bx),\sum_{i=1}^n  \hat\rho^+_i} - \norm{a_-}_{\M} \sp{\sigma(Bx),\sum_{i=1}^n  \hat\rho^-_i}
\end{equation*}
and using the triangle inequality, we finally get the following estimate 
\begin{eqnarray*}
    d_*(f(x),\hat f_n(x)) &\leq& d_*(f^+(x),\hat f^+_n(x)) + d_*(f^-(x),\hat f^-_n(x)) \\
    &\leq& (1+\eps)\frac{\norm{x}(\norm{a_+}_{\M} + \norm{a_-}_{\M}) }{\sqrt{n}}\\
    &=& (1+\eps)\frac{\norm{x}\norm{a}_{\M} }{\sqrt{n}} \\
    &=& (1+\eps)\frac{\norm{x}\norm{f}_{\Barr} }{\sqrt{n}}.
\end{eqnarray*}
The function $\hat f_n(x)$, in fact, contains $2n$ terms. Replacing $n$ with $n/2$ and letting $c \defeq (1+ \eps)\sqrt{2}>2\sqrt{2}$, we get the desired estimate.

\item Convergence rates in Bochner spaces can be obtained similarly to Theorem~\ref{thm:inverse_approximation}.
\end{enumerate}

\end{proof}

\begin{remark}
The fact that the operators $K_i$ in Theorems~\ref{thm:inverse_approximation} and~\ref{thm:direct_approximation} are finite-rank is important from the computational point of view, since any operator we can realise on a computer will be finite-rank. 
\end{remark}

\begin{remark}
Monte-Carlo rates in Lebesgue spaces $L^2_\pi(\X)$ were obtained in~\cite{e2019barron, e2020barron-representation} and Theorem~\ref{thm:direct_approximation} extends them to Bochner spaces $L^2_\pi(\X;(\Y,d_*))$. In the case $2 < p < \infty$, suboptimal rates $\bigO(n^{-1/p})$ are obtained using an interpolation argument in~\cite{e2020barron-representation} (discussion immediately after Theorem 3.8); our Theorem~\ref{thm:direct_approximation} improves on these results even in the scalar-valued case by providing Monte-Carlo rates for all $p \in [1,\infty)$.
\end{remark}

%%%%%%%%%%%%%%%%
\subsubsection{Examples}\label{sec:examples}

As already pointed out in Remark~\ref{rem:lattice-op-ws-cont}, lattice operations are rarely sequentially continuous in the weak or weak-* topologies. Therefore, Theorems~\ref{thm:inverse_approximation} and~\ref{thm:direct_approximation} make a rather strong assumption on the output space $\Y$. 

\paragraph{Atomic Banach lattices.} If $\Y$ is a Banach lattice with an order continuous norm then lattice operations are weakly sequentially continuous if and only if $\Y$ is purely atomic~\cite[Cor. 2.3]{chen:1998-lattice}.
\begin{definition}[Purely atomic Banach lattice]
Let $E$ be a Banach lattice. An element $0 \neq a \in E_+$ is called an atom if for any $b \in E$ such that $0 \leq b \leq a$ it holds $b = \lambda a$ for some $\lambda \in \R$. A Banach lattice is called \emph{purely atomic} if the only element that is disjoint from all atoms is the zero element.
\end{definition}
The situation with sequential weak-* continuity of lattice operations is an even more delicate matter, and purely atomic dual spaces may fail to have weakly-* sequentially continuous lattice operations~\cite[\S 3]{chen:1998-Rademacher}.

Clearly, sequence spaces $\ell^p$, $1 < p < \infty$, are purely atomic and reflexive, hence lattice operations are weakly-* sequentially continuous. Sequential weak-* continuity (and, indeed, $1$-Lipschitzness under a certain choice of the basis in the predual) of lattice operations in $\ell^\infty$ was shown in Proposition~\ref{prop:ell-p-lattice-cont}. Lattice operations in $\ell^1$ (considered as the dual of the space $c$ of convergent sequences) are not sequentially continuous in the weak-* topology~\cite[\S 3]{chen:1998-Rademacher}.

Another example of purely atomic Banach lattices are $\mathcal L^p(\mu)$ spaces, $1 < p < \infty$, where the measure $\mu$ is purely atomic.

\paragraph{Lipschitz spaces.} 
If $\Y$ is a Riesz space but not a Banach lattice then the result of~\cite[Cor. 2.3]{chen:1998-lattice} does not apply and atomicity is not required. 

Let $\Omega$ be a complete metric space of finite diameter and let $\Y = \Lip(\Omega)$ be the space of Lipschitz functions on $\Omega$ with the following norm
\begin{equation*}
    \norm{f}_{\Lip} \defeq \max\{Lip(f), \, \norm{f}_\infty\},
\end{equation*}
where $Lip(f)$ denotes the Lipschitz constant of $f$, and the pointwise partial order 
\begin{equation*}
f \geq g \iff f(x) \geq g(x) \quad \text{for all $x \in \Omega$}.
\end{equation*}

\begin{theorem}[{\cite[Thm. 3.23 and Cor. 3.4]{weaver:2018}}]\label{thm:Lip-predual-unique}
Let $\Omega$ be a complete metric space of finite diameter. Then the space $\Lip(\Omega)$ has a unique predual. On bounded subsets of $\Lip(\Omega)$ its weak-* topology coincides with the topology of pointwise convergence.
\end{theorem}

The next result shows that lattice operations are sequentially weakly-* continuous in Lipschitz spaces.

\begin{theorem}\label{thm:lattice-op-Lip}
Let $\Omega$ be a complete metric space of finite diameter. Then lattice operations in $\Lip(\Omega)$ are sequentially weakly-* continuous.
\end{theorem}
\begin{proof}
Let $f_n \wsto f$ weakly-* in $\Lip(\Omega)$. Since weakly-* convergent sequences are bounded, we get from Theorem~\ref{thm:Lip-predual-unique} that convergence is pointwise, i.e. 
\begin{equation*}
f_n(x) \to f(x) \quad \text{for all $x \in \Omega$}.
\end{equation*}
Consequently, 
\begin{equation*}
(f_n)_+(x) \to f_+(x) \quad \text{for all $x \in \Omega$}.
\end{equation*}
Since $\norm{f_+}_{\Lip} \leq \norm{f}_{\Lip}$ for any $f \in \Lip(\Omega)$~\cite{weaver:2018}, the sequence $(f_n)_+$ is bounded and therefore $(f_n)_+ \wsto f_+$ weakly-* in $\Lip(\Omega)$.
\end{proof}
\begin{remark}
As we can see from the proof, lattice operations are weakly-* continuous in $\Lip(\Omega)$ because they are continuous in $\R$. With an appropriate choice of the dense system in the predual (cf. Proposition~\ref{prop:rank-one-dense}) we will also get that lattice operations are $1$-Lipschitz with respect to the $d_*$ metric~\eqref{eq:ws-metric-Y} in $\Lip(\Omega)$.
\end{remark}

We conclude that the choice $\Y = \Lip(\Omega)$ satisfies the conditions of Theorems~\ref{thm:inverse_approximation} and~\ref{thm:direct_approximation}.

%%%%%%--------------------------------------%%%%%%%

%%%%%%%%%%%%%%%%
\section{\revone{Finding an optimal representation}}
\label{sec:training}

\revone{
It has been shown (e.g.~\cite{parhi:2021}) that training a sufficiently wide (scalar-valued) two-layer neural network with weight decay~\cite{krogh:1991} results in functions which are optimal with respect to the $\Barr$ norm. These functions solve the following variational problem
\begin{equation}\label{eq:training}
    \min_{g \in \Barr} \frac1{pm}\sum_{i=1}^m \abs{g(x_i)-f(x_i)}^p + \lambda \norm{g}_{\Barr},
\end{equation}
where $\{(x_i,f(x_i))\}_{i=1}^m$ are the training data (possibly, corrupted by noise) and $\lambda>0$ is a regularisation parameter. Using the representation $g = \sp{\sigma(B\cdot),a}$ valid for any $\Barr$ function, one can rewrite~\eqref{eq:training} as an optimisation problem over the space of measures, cf.~\eqref{eq:var-prob}. For a fixed data fitting term (i.e., fixed $m$), this is a well-studied problem, e.g.~\cite{bredies2013measures}. The contribution of this section compared to~\cite{bredies2013measures} is to study the consistency of minimisers as the number of samples $m \to \infty$. We refer to \cref{rem:setting-var-reg} for a discussion of our setting in the context of existing literature.
}

Given a function $f \in \Barr({{\X}};\Y)$ (whose explicit form may be unknown), consider the problem of finding the best neural network representation of $f$, i.e. a representation with the smallest {variation norm}. 
Define the following operator $T$ acting on $\M(\Ball{\calL})$
\begin{equation}\label{eq:T}
    Ta \defeq \sp{\sigma(B\cdot),a}, \quad a \in \M(\Ball{\calL}).
\end{equation}
By Remark~\ref{rem:Barr-in-Lp} we  have that $T$ is continuous as an operator $\M(\Ball{\calL}) \to L^p_\pi(\X;(\Y,d_*))$, where $\pi \in \P_p(\X)$ is a probability measure with $p \geq 1$ finite moments.

Finding the best representation~\eqref{eq:NN_dual_pairing} of $f$  accounts to finding the minimum norm solution of the following inverse problem
\begin{equation}\label{eq:IP}
    Ta=f.
\end{equation}

\begin{definition}[Minimum-norm solution]
A solution $\Jminsol$ of~\eqref{eq:IP} is called a minimum-norm solution if
\begin{equation*}
    \norm{\Jminsol}_{\M} \leq \norm{a}_{\M} \quad \text{for all $a$ such that $Ta=f$}.
\end{equation*}
\end{definition}

The following result shows that $T$ is a compact operator $\M(\Ball{\calL}) \to L^p_\pi(\X;(\Y,d_*))$. 
\begin{prop}\label{prop:T-compact}
The operator $T \colon a \mapsto \sp{\sigma(B\cdot),a}$ is compact as an operator $\M(\Ball{\calL}) \to L_\pi^p(\X;(\Y,d_*))$.
\end{prop}
\begin{proof}
The proof can be found in Appendix~\ref{app:proofs}.
\end{proof}

%%%%%%%%%%%%%%%%

Since the inverse of a compact operator is unbounded,~\eqref{eq:IP} is ill-posed and regularisation is required. 
In this section, we will study variational regularisation of~\eqref{eq:IP}. In particular, we will investigate the role of a further regularity condition on the function $f \in \Barr(\X;\Y)$, known as the \emph{source condition}, and its relation to convergence in Bregman distance and support recovery.

\revone{We will assume that  $\{x_i\}_{i=1}^m \subset \X$ are i.i.d. samples $x_i \sim \pi$ from a probability measure $\pi$ on $\X$. We will also assume the exact data $\{f(x_i)\}_{i=1}^m$ are corrupted by additive noise $\{e_i\}_{i=1}^m \subset \Y$ such that
\begin{equation}\label{eq:norm-est-error}
    \mean_\pi \frac1m \sum_{i=1}^m \norm{e_i}^p_{(\Y,d_*)} \leq \eps^p \quad \text{and} \quad \mean_\pi \frac1m \sum_{i=1}^m \norm{e_i}^{2p}_{(\Y,d_*)} \leq \sigma^2 < \infty
\end{equation}
uniformly in $m$ for some $\eps,\sigma > 0$.
} 
We end up with the following variational problem, cf.~\eqref{eq:training}
\begin{equation}\label{eq:var-prob}
    \min_{a \in \M(\Ball{\calL})} \frac1{pm} \sum_{i=1}^m \norm{Ta(x_i) - f(x_i) + e_i}^p_{(\Y,d_*)} + \lambda \norm{a}_{\M}.
\end{equation}
Existence of a minimiser for any fixed $\lambda>0$ follows from standard arguments, e.g.,~\cite{Benning_Burger_modern:2018}. We denote this minimiser by $a_{m,\eps}$. Since~\eqref{eq:var-prob} is not strictly convex, this minimiser is not unique. All consequent results will hold for any minimiser  of~\eqref{eq:var-prob}. 

\revone{
\begin{remark}\label{rem:setting-var-reg}
We will study the behaviour of minimisers in the limit as $\eps \to 0$ and $m \to \infty$. This is standard in inverse problems literature~\cite{scherzer_var_meth:2009,Benning_Burger_modern:2018}, but different from the setting of statistical learning, where $\eps$ is assumed fixed and only $m \to \infty$. Due to the ill-posedness of~\eqref{eq:IP}, convergence of statistical estimators can be arbitrary slow (e.g.,~\cite{krzyzak:2002}) unless regularisation is used. (We emphasise that this applies to estimators of the representing measure $\Jminsol$, not the function $f$ itself.) For spectral regularisation, convergence rates in the statistical setting were obtained in~\cite{blanchard:2018}. Rates for variational regularisation with convex $p$-homogeneous functionals, $p>1$, were obtained in a recent paper~\cite{bubba2021stat}. Our case of a $1$-homogeneous functional $\norm{\cdot}_{\M}$ is not covered by existing literature. Obtaining convergence rates for solutions of~\eqref{eq:var-prob}  with fixed $\eps$ in the spirit of~\cite{blanchard:2018,bubba2021stat} is beyond the scope of the present paper.
\end{remark}
}

\begin{prop}\label{prop:est-norm-ameps}
\revtwo{
The following statements about a minimiser $a_{m,\eps}$  of~\eqref{eq:var-prob} hold
\begin{enumerate}[(i)]
    \item $\mean_\pi \norm{a_{m,\eps}}_{\M} \leq \frac{\eps^p}{p\lambda} + \norm{\Jminsol}_{\M}$;
    \item $\mean_\pi \norm{a_{m,\eps}}^2_{\M} \leq 2\left( \frac{\sigma^2}{p^2\lambda^2} + \norm{\Jminsol}^2_{\M} \right)$;
    \item for any $0<\delta<1$,
    \begin{equation*}
        \norm{a_{m,\eps}}_{\M} \leq \frac{\eps^p}{p\lambda} + \norm{\Jminsol}_{\M} + \frac{\sqrt{2}}\delta \left(\frac{\sigma}{p\lambda} + \norm{\Jminsol}_{\M}\right)
    \end{equation*}
    with probability at least $1-\delta^2$.
\end{enumerate}
}
\end{prop}
\begin{proof}
\revtwo{
\begin{enumerate}[(i)]
    \item Using the optimality of $a_{m,\eps}$ in~\eqref{eq:var-prob} and comparing the value of the objective with the value at \revtwo{any minimum norm solution $\hat a$}, we get (recall that $T\hat a=f$)
\begin{equation}\label{eq:ameps-Jminsol-compare}
    \frac1{pm} \sum_{i=1}^m \norm{Ta_{m,\eps}(x_i) - f(x_i) + e_i}^p_{(\Y,d_*)} + \lambda \norm{a_{m,\eps}}_{\M} \leq \frac1{pm} \sum_{i=1}^m \norm{e_i}^p_{(\Y,d_*)} + \lambda \norm{\hat a}_{\M},
\end{equation}
which also holds if we take the expected value over $\{x_i\}_{i=1}^m \sim \pi$. Therefore, we get
\begin{equation}
    \mean_\pi \norm{a_{m,\eps}}_{\M} \leq \frac1{p\lambda} \mean_\pi \frac1{m} \sum_{i=1}^m \norm{e_i}^p_{(\Y,d_*)} + \norm{\hat a}_{\M} \\
    \leq \frac{\eps^p}{p\lambda}  + \norm{\hat a}_{\M}.
\end{equation}
\item Similarly, from~\eqref{eq:ameps-Jminsol-compare} we get
\begin{eqnarray*}
    \norm{a_{m,\eps}}^2_{\M} &\leq& \left(\frac1{p\lambda} \frac1m \sum_{i=1}^m \norm{e_i}^p_{(\Y,d_*)} + \norm{\hat a}_{\M}\right)^2 \leq 2 \left( \frac{1}{p^2\lambda^2} \frac1m \sum_{i=1}^m \norm{e_i}^{2p}_{(\Y,d_*)} + \norm{\Jminsol}^2_{\M} \right),
\end{eqnarray*}
where we used Young's and Jensen's inequalities. Taking the expectation, we obtain the claim.
\item This follows from Chebyshev's inequality.
\end{enumerate}
}
\end{proof}

%%%%%%%%%%%%%%
\subsection{Source condition}

In order to characterise the convergence of the solutions of~\eqref{eq:var-prob} in terms of convergence rates, one needs to make additional assumptions on the regularity of $f$. We are particularly interested in a \emph{structural} characterisation of this convergence in terms of the convergence of the support of the measure,  which is usually done using Bregman distances.
We will now introduce some necessary concepts.

\begin{definition}[Subdifferential]
Let $E$ be a Banach space and $\reg \colon E \to \R \cup \{+\infty\}$ a proper convex \lsc{} functional. An element $\subgradA \in E^*$ is called a subgradient of $\reg$ at $x_0 \in E$ if 
\begin{equation*}
    \reg(x) \geq \reg(x_0) + \sp{\subgradA,x-x_0} \quad \forall x \in E.
\end{equation*}
The collection of all subgradients at $x_0$ is called the subdifferential of $\reg$ at $x_0$ and denoted by
\begin{equation*}
    \dJ(x_0) \defeq \{\subgradA \in E^* \colon \text{$\subgradA$ is a subgradient of $\reg$ at $x_0$}\}.
\end{equation*}
\end{definition}

\begin{example}[Subgradients of the Radon norm] \label{ex:subgrad-Radon}
Consider $\reg(\mu) \defeq \norm{\mu}_{\M}$ for $\mu \in \M(\Omega)$, where $\Omega$ is compact. Let $\mu_0 \in \M(\Omega)$. It is well known (e.g.,~\cite{bredies2013measures}) that $\subgradA_0 \in \dRadon(\mu_0)$ if and only if
\begin{equation}\label{eq:one-hom}
    \norm{\subgradA_0}_{\M^*} \leq 1 \quad \text{and} \quad \norm{\mu_0}_{\M} = \sp{\subgradA_0,\mu_0}.
\end{equation}
Applying the Cauchy-Schwarz inequality, we get that
\begin{equation*}
    \norm{\mu_0}_{\M} \leq \norm{\subgradA_0}_{\M^*}\norm{\mu_0}_{\M} \leq \norm{\mu_0}_{\M}
\end{equation*}
and hence the Cauchy-Schwarz inequality holds as an equality. For a sufficiently regular subgradient $\subgradA_0 \in \C(\Omega)$, this implies that 
\begin{equation*}
    \subgradA_0 = \sign(\mu_0) \quad \text{on $\supp(\mu_0)$}.
\end{equation*}
\end{example}

Since~\eqref{eq:IP} is ill-posed, convergence of the minimisers of~\eqref{eq:var-prob} to a minimum-norm solution may be arbitrary slow unless a further regularity assumption on $\Jminsol$, called the \emph{source condition}, is made. Depending on the situation, different variants of the source condition exist. We use the definition from~\cite{Burger_Osher:2004}.
\begin{assumption}[Source condition]\label{ass:sc}
The minimum-norm solution $\Jminsol$ of~\eqref{eq:IP} satisfies the source condition, i.e. there exists an $\omega^\dagger \in (L^p_\pi({{\X}};(\Y,d_*)))^* = L^q_\pi(\X;((\Y,d_*))^*)$ such that
\begin{equation*}
    T^*\omega^\dagger \in \dRadon(\Jminsol),
\end{equation*}
where $T^* \colon L^q_\pi(\X;((\Y,d_*))^*) \mapsto \M^*(\Ball{\calL({{\X}};\Y)})$ is the adjoint of $T$.
\end{assumption}

To understand the implications of Assumption~\ref{ass:sc}, we need to study the adjoint operator $T^*$. It is easy to show using Fubini's theorem that for any $v \in L^q_\pi(\X;((\Y,d_*))^*)$ ($q$ is the conjugate exponent of $p$)
\begin{equation}\label{eq:K-adj}
    T^*v = \int_\X \sp{\sigma(Bx),v(x)} \, d\pi(x) \in \M^*(\Ball{\calL({{\X}};\Y)}),
\end{equation}
where the pairing is between $(\Y,d_*)$ and $((\Y,d_*))^*$, and for any $K \in \Ball{\calL({{\X}};\Y)}$ we have
\begin{equation}\label{eq:K-adj-eval-at-N}
    (T^*v)(K) = \int_\X \sp{(Kx)_+,v(x)} \, d\pi(x),
\end{equation}
where $(\cdot)_+$ is the operation of taking the positive part in $\Y$. At this point,~\eqref{eq:K-adj-eval-at-N} is a formal expression, since we don't know whether $T^*v$ is  continuous on $\Ball{\calL(\X;\Y)}$ with respect to weak-* convergence. The next result shows that this is the case. Hence, as in the previous section, $T^*$ maps into a more regular space than $\M^*(\Ball{\calL({{\X}};\Y)})$ -- its double predual $\C(\Ball{\calL({{\X}};\Y)})$.

\begin{prop}\label{prop:K-adj-Lip}
Let $\pi \in \P_p(\X)$ have $p \geq 1$ finite moments. Suppose that lattice operations in $\Y$ are $1$-Lipschitz with respect to the $d_*$ metric. Then $\forall v \in L^q_\pi(\X;((\Y,d_*))^*)$
\begin{equation}
    T^* v \in \C(\Ball{\calL({{\X}};\Y)}),
\end{equation}
where continuity is with respect to the weak-* convergence in $\Ball{\calL(\X;\Y)}$. Moreover, if we consider the unit ball $\Ball{\calL(\X;\Y)}$ with the strong (norm-based) metric, then  $T^* v$ is Lipschitz and its Lipschitz constant is bounded by
\begin{equation*}
    \Lip(T^*v) \leq (m_p(\pi))^{\frac1p}\norm{v}_{L^q_\pi},
\end{equation*}
where $m_p(\pi)$ is the $p$-th moment of $\pi$.
\end{prop}
\begin{proof}
Let $K_n \wsto K$ weakly-* in $\Ball{\calL({{\X}};\Y)}$. By~\eqref{eq:K-adj-eval-at-N} we have \begin{eqnarray*}
    \abs{(T^*v)(K_n) - (T^*v)(K)} &=& \abs{\int_\X \sp{(K_n x)_+ - (Kx)_+, v(x)} \, d\pi(x)} \\
    &\leq& \int_\X \norm{(K_n x)_+ - (K x)_+}_{(\Y,d_*)} \norm{v(x)}_{((\Y,d_*))^*} \, d\pi(x) \\
    &\leq& \norm{(K_n \cdot)_+ - (K\cdot)_+}_{L^p_\pi} \norm{v}_{L^q_\pi}.
\end{eqnarray*}
Now, for any $x \in \X$ we have 
\begin{eqnarray*}
    \norm{(K_n x)_+ - (Kx)_+}_{(\Y,d_*)} &=& d_*((K_n x)_+ - (Kx)_+,0) = d_*((K_n x)_+, (Kx)_+) \leq d_*(K_n x, Kx),
\end{eqnarray*}
since lattice operations are $1$-Lipschitz with respect to the $d_*$ metric in $\Y$. 

\revtwonew{We also have that for any $K_1,K_2 \in \Ball{\calL({{\X}};\Y)}$ and any $x \in \X$
\begin{equation*}
    d_*(K_1x,K_2x) \leq \norm{K_1x-K_2x}_{\Y} \leq \norm{K_1-K_2}_{\calL(\X;\Y)}\norm{x}_{\X}
\end{equation*}
and therefore
\begin{equation*}
    d_*(K_n x, Kx)^p \leq \norm{K_n-K}^p \norm{x}^p \leq 2^p \norm{x}^p 
\end{equation*}
for all $n$. Since $\pi$ has $p$ finite moments, the right-hand side is integrable and, using the reverse Fatou lemma, we get  that
}
\begin{eqnarray*}
    \limsup_{n \to \infty} \norm{(K_n \cdot)_+ - (K\cdot)_+}^p_{L^p_\pi} &=& \limsup_{n \to \infty} \int_\X \norm{(K_n x)_+ - (K x)_+}^p_{(\Y,d_*)} \, d\pi(x) \\
    &\leq& \limsup_{n \to \infty} \int_\X (d_*(K_n x, Kx))^p \, d\pi(x) \\
    &\leq& \int_\X \limsup_{n \to \infty} (d_*(K_n x, Kx))^p \, d\pi(x) \\
    &=& 0,
\end{eqnarray*}
where we  used the fact that $K_n x \wsto Kx$ in $\Y$ for all $x \in \X$ if $K_n \wsto K$ in $\Ball{\calL({{\X}};\Y)}$. Therefore, $\abs{(T^*v)(K_n) - (T^*v)(K)} \to 0$ and consequently $T^* v \in \C(\Ball{\calL({{\X}};\Y)})$.

To prove the second claim, we observe that 
\begin{equation*}
    \int_\X (d_*(K_1 x, K_2 x))^p \, d\pi(x) \leq \norm{K_1-K_2}^p_{\calL(\X;\Y)} \int_\X \norm{x}^p_{\X} \, d\pi(x) = \norm{K_1-K_2}^p_{\calL(\X;\Y)} m_p(\pi)
\end{equation*}
and 
\begin{equation*}
    \abs{(T^*v)(K_1) - (T^*v)(K_2)} \leq (m_p(\pi))^{\frac1p} \norm{v}_{L^q_\pi} \norm{K_1-K_2}_{\calL(\X;\Y)}.
\end{equation*}
\end{proof}

Now we can give an intuitive interpretation of the source condition (Assumption~\ref{ass:sc}). Let $\Jminsol_+$ and $\Jminsol_-$ be the positive and negative parts of $\Jminsol$ in the sense of the Hahn-Jordan decomposition. 
From Example~\ref{ex:subgrad-Radon} we know that for any subgadient $\subgradA \in \dRadon(\Jminsol) \cap \C(\Ball{\calL({{\X}};\Y)})$
\begin{equation}\label{eq:separation}
    \left. \subgradA \right|_{\supp(\Jminsol_+)} \equiv 1 \quad \text{and} \quad \left. \subgradA \right|_{\supp(\Jminsol_-)} \equiv -1.
\end{equation}
If $\Jminsol$ satisfies Assumption~\ref{ass:sc}, the subgradient $T^*\omega^\dagger$ is continuous with respect to the $d_*$ metric, hence it is not allowed to jump form $-1$ to $1$ without passing all values inbetween.
In other words, $\Jminsol$ is not allowed to change sign without staying zero on some interval. 

Since with respect to the strong (norm-based) metric  $T^*\omega^\dagger$ is  Lipschitz on $\Ball{\calL(\X;\Y)}$, we have a lower bound on the length of this interval (in the norm-induced metric on $\Ball{\calL(\X;\Y)}$).

\begin{prop}
Let $\Jminsol$ satisfy Assumption~\ref{ass:sc} and denote by $\Jminsol_+$ and $\Jminsol_-$ the positive and negative parts of $\Jminsol$ (in the sense of the Hahn-Jordan decomposition). Then for any $K \in \supp(\Jminsol_+)$ and $K' \in \supp(\Jminsol_-)$ we have
\begin{equation*}
    \norm{K-K'}_{\calL(\X;\Y)} \geq \frac{2}{(m_p(\pi))^{\frac1p} \norm{\omega^\dagger}_{L^q_\pi}},
\end{equation*}
where $\omega^\dagger$ is as in Assumption~\ref{ass:sc} and $m_p(\pi)$ is the $p$-th moment of $\pi$.
\end{prop}
\begin{proof}
From Example~\ref{ex:subgrad-Radon} we know that
\begin{equation*}
    T^*\omega^\dagger \equiv 1 \quad \text{on $\supp(\Jminsol_+)$} \quad \text{and} \quad T^*\omega^\dagger \equiv -1 \quad \text{on $\supp(\Jminsol_-)$}.
\end{equation*}
Taking any $K \in \supp(\Jminsol_+)$ and $K' \in \supp(\Jminsol_-)$, we get that
\begin{equation*}
    \abs{T^*\omega^\dagger(K) - T^*\omega^\dagger(K')} = 2.
\end{equation*}
Using the upper bound on the Lipschitz constant of $T^*\omega^\dagger$ from Proposition~\ref{prop:K-adj-Lip}, we get
\begin{equation*}
    (m_p(\pi))^{\frac1p} \norm{\omega^\dagger}_{L^q_\pi} \geq \Lip(T^*\omega) \geq \frac{\abs{T^*\omega^\dagger(K) - T^*\omega^\dagger(K')}}{\norm{K-K'}_{\calL(\X;\Y)}} = \frac{2}{\norm{K-K'}_{\calL(\X;\Y)}},
\end{equation*}
which implies the claim.
\end{proof}

In terms of the representation~\eqref{eq:NN_dual_pairing} this means that the network is not allowed to have two neurons with opposite signs and weights that are too close to each other, the distance being controlled by the source element $\omega^\dagger$. Hence, intuitively, {$\Barr$ function}s satisfying the source condition are more stable in the sense that cancellations are avoided, and this effect is the stronger the smaller the norm of the source element.

\begin{remark}
The inequality~\eqref{eq:separation} is a separation condition for the supports of the positive and the negative parts of the measure $\Jminsol$ in terms of the strong (norm-based) metric on $\Ball{\calL(\X;\Y)}$. No such separation can be obtained in the weak-* metric~\eqref{eq:ws-metric-op}. However, other metrics  might admit a separation condition. For example,~\cite{poon:2019-support} argues that the Fisher-Rao metric is a good choice for studying support separation. Such questions are outside the scope of the present paper.
\end{remark}

%%%%%%%
\subsection{Convergence rates in Bregman distance}

In modern variational regularisation, (generalised) Bregman distances are typically used to study convergence rates of solutions~\cite{Benning_Burger_modern:2018}.
\begin{definition}[Bregman distance]\label{def:Bregman-distance}
Let $E$ be a Banach space and $\reg \colon E \to \R \cup \{+\infty\}$ be a proper convex functional. The generalised Bregman distance between $u,u' \in E$ corresponding to the subgradient $\subgradA' \in \dJ(u')$ is defined as follows
	\begin{equation*}
	\D{\reg}{\subgradA'}(u,u') \defeq \reg(u) - \reg(u') - \sp{\subgradA',u-u'},
	\end{equation*}
	where $\dJ(u')$ denotes the subdifferential of $\reg$ at $u' \in E$. The symmetric Bregman distance between $u$ and $u'$ corresponding to $\subgradA \in \dJ(u)$ and $p' \in \dJ(u')$ is defined as follows
	\begin{equation*}
	\Dsymm{\reg}(u,u') \defeq \D{\reg}{\subgradA'}(u,u') + D_\reg^{\subgradA}(u',u) = \sp{\subgradA-\subgradA',u-u'}.
	\end{equation*}
\end{definition}
Bregman distances do not define a metric, since they do not satisfy the triangle inequality and $\Dsymm{\reg}(u,u') = 0$ does not imply $u=u'$ in general.

\begin{example}[Bregman distance of the Radon norm]\label{ex:Bregman-distance-Radon}
Consider $\reg(\mu) \defeq \norm{\mu}_{\M}$ for $\mu \in \M(\Omega)$, where $\Omega$ is compact. Let $\mu_0 \in \M(\Omega)$ and $\subgradA_0 \in \dRadon(\mu_0)$. Then for any $\mu \in \M(\Omega)$, the Bregamn distance between $\mu_0$ and $\mu$ corresponding to the subgradient $\subgradA_0$ is given by
\begin{equation*}
    \D{\norm{\cdot}_{\M}}{\subgradA_0}(\mu,\mu_0) = \norm{\mu}_{\M} - \sp{\subgradA_0,\mu}.
\end{equation*}
Zero Bregman distance would imply that $\subgradA_0 \in \dRadon(\mu)$. In light of Example~\ref{ex:subgrad-Radon}, this means that $\subgradA_0 = \sign(\mu)$ on $\supp(\mu)$. Hence, roughly speaking, $\D{\norm{\cdot}_{\M}}{\subgradA_0}(\mu,\mu_0)$ measures the deviation of $\supp(\mu)$ from $\supp(\mu_0)$.
\end{example}

We will need the following result.
\begin{lemma}\label{lem:Barron-variance}
\revtwo{
Let $\pi \in \P_{2p}(\X)$ be a probability measure with $2p$ finite moments and $f \in \Barr(\X;\Y)$. Then for any $0<\delta<1$
\begin{equation*}
    \left| \frac1m \sum_{i=1}^m \norm{f(x_i)}^p_{(\Y,d_*)} - \norm{f}^p_{L^p_\pi} \right| \leq \frac1\delta \frac{\norm{f}^p_{\Barr}\sqrt{m_{2p}(\pi)}}{\sqrt{m}},
\end{equation*}
with probability at least $1-\delta^2$, where $m_{2p}(\pi)<\infty$ is the $2p$-th moment of $\pi$.
}
\end{lemma}
\begin{proof}
\revtwo{
Since $f$ is Lipschitz (with Lipschitz constant bounded by $\norm{f}_{\Barr}$, see \cref{prop:NN_Lipschitz}), it grows at most linearly at infinity, and we get 
\begin{eqnarray*}
    var_\pi \left(\frac1m \sum_{i=1}^m \norm{f(x_i)}^p_{(\Y,d_*)} \right) &=& \frac1{m^2} \sum_{i=1}^m var_\pi \left(\norm{f(x_i)}^p_{(\Y,d_*)} \right) \leq  \frac1m \mean_\pi\left[ \norm{f(x_1)}^{2p}_{(\Y,d_*)} \right] \\ 
    &=&  \frac1m\int_\X \norm{f(x)}^{2p}_{(\Y,d_*)} \, d\pi(x) \leq \frac{\norm{f}^{2p}_{\Barr}}{m}  \int_\X \norm{x}^{2p} \, d\pi(x) = \frac{\norm{f}^{2p}_{\Barr} m_{2p}(\pi)}{m}.
\end{eqnarray*}
The claim now follows from Chebyshev's inequality.
}
\end{proof}

\begin{theorem}\label{thm:conv-rates-Bregman}
\revtwo{
 Let $\pi \in \P_{2p}(\X)$ be a probability measure with $2p$ finite moments and suppose that a minimum-norm solution $\Jminsol$ satisfies the source condition (Assumption~\ref{ass:sc}). Then, up to multiplicative constants and higher-order terms in $\lambda$ and $\eps$, the following estimate holds for any $0<\delta<1$ with probability at least $1-3\delta^2$
\begin{equation*}
    D^{\subgradA^\dagger}_{\norm{\cdot}_{\M}} (a_{m,\eps}, \Jminsol) \leq 
    \begin{cases}
    \frac{\eps^p}{\lambda}  + \frac1{\delta^p\lambda^{p+1}\sqrt{m}} + \lambda^{q-1} \norm{\omega^\dagger}^q_{L^q_\pi}, \quad &\text{if $p > 1$},\\
    \frac{\eps^p}{\lambda}  + \frac1{\delta^p\lambda^{p+1}\sqrt{m}} \quad &\text{if $p=1$ and $\lambda \leq \frac{1}{\norm{\omega^\dagger}}$},
    \end{cases}
\end{equation*}
where $a_{m,\eps}$ is a solution of~\eqref{eq:var-prob}, $q$ is the H\"older conjugate of $p$, $\omega^\dagger$ is the source element from Assumption~\ref{ass:sc}, and $\subgradA^\dagger \defeq T^*\omega^\dagger$. 
}
\end{theorem}
\begin{proof}
\revtwo{
The proof is similar to, e.g.,~\cite[Thm. 3.5]{LB_MB_YK_CBS:2020}. Subtracting $\lambda\sp{\subgradA^\dagger,a_{m,\eps}}$ from both sides of~\eqref{eq:ameps-Jminsol-compare} and rearranging terms, we get 
\begin{eqnarray*}
    \lambda D^{\subgradA^\dagger}_{\norm{\cdot}_{\M}} (a_{m,\eps}, \Jminsol) &\leq&  \frac1{pm} \sum_{i=1}^m \norm{e_i}^p_{(\Y,d_*)} - \frac1{pm} \sum_{i=1}^m \norm{Ta_{m,\eps}(x_i) - f(x_i) + e_i}^p_{(\Y,d_*)} + \lambda \norm{\hat a}_{\M} - \lambda \sp{\subgradA^\dagger,a_{m,\eps}}. % \\
    %  &\leq& \frac{1+2^{p-1}}{pm} \sum_{i=1}^m \norm{e_i}^p_{(\Y,d_*)} - \frac{2^{p-1}}{pm} \sum_{i=1}^m \norm{Ta_{m,\eps}(x_i) - f(x_i)}^p_{(\Y,d_*)} + \lambda \sp{\omega^\dagger,f-T a_{m,\eps}},
\end{eqnarray*}
\revtwonew{Using the inequality $\norm{g-g'}^p \leq 2^{p-1} \left( \norm{g}^p + \norm{g'}^p \right)$, which is valid for the convex and absolutely $p$-homogeneous  function $\norm{\cdot}^p$ with $p \geq 1$, we get that
\begin{equation*}
    2^{1-p} \norm{Ta_{m,\eps}(x_i) - f(x_i)}^p_{(\Y,d_*)} \leq  \norm{Ta_{m,\eps}(x_i) - f(x_i) + e_i}^p_{(\Y,d_*)} + \norm{e_i}^p_{(\Y,d_*)}.
\end{equation*}
Therefore, we have the following estimate
\begin{eqnarray*}
    \lambda D^{\subgradA^\dagger}_{\norm{\cdot}_{\M}} (a_{m,\eps}, \Jminsol) \leq \frac2{pm} \sum_{i=1}^m \norm{e_i}^p_{(\Y,d_*)} - \frac{2^{1-p}}{pm} \sum_{i=1}^m \norm{Ta_{m,\eps}(x_i) - f(x_i)}^p_{(\Y,d_*)} + \lambda \sp{\omega^\dagger,f-T a_{m,\eps}}.
\end{eqnarray*}
} \\
By \cref{lem:Barron-variance} we have that 
\begin{eqnarray*}
    - \frac{1}{pm} \sum_{i=1}^m \norm{Ta_{m,\eps}(x_i) - f(x_i)}^p_{(\Y,d_*)} &\leq& -\frac1p\norm{Ta_{m,\eps} - f}^p_{L^p_\pi} + \frac1{\delta} \frac{\left(\norm{a_{m,\eps}}_{\M} + \norm{\Jminsol}_{\M} \right)^p \sqrt{m_{2p}(\pi)}}{\sqrt{m}} \\
    &\leq& -\frac1p\norm{Ta_{m,\eps} - f}^p_{L^p_\pi} + \frac{2^{p-1}}{\delta} \frac{\left(\norm{a_{m,\eps}}^p_{\M} + \norm{\Jminsol}^p_{\M} \right) \sqrt{m_{2p}(\pi)}}{\sqrt{m}} 
\end{eqnarray*}
with probability at least $1-\delta^2$. \revtwonew{Therefore,
\begin{eqnarray*}
     - \frac{2^{1-p}}{pm} \sum_{i=1}^m \norm{Ta_{m,\eps}(x_i) - f(x_i)}^p_{(\Y,d_*)} &+& \lambda \sp{\omega^\dagger,f-T a_{m,\eps}} \leq \frac{1}{\delta} \frac{\left(\norm{a_{m,\eps}}^p_{\M} 
     + \norm{\Jminsol}^p_{\M} \right) \sqrt{m_{2p}(\pi)}}{\sqrt{m}} \\
     &+& 2^{1-p} \left[\sp{2^{p-1}\lambda\omega^\dagger,f-T a_{m,\eps}} -\frac1p\norm{f - Ta_{m,\eps}}^p_{L^p_\pi} \right]
\end{eqnarray*}
with the same probability. By the Fenchel-Young inequality we have that
\begin{equation*}
    \sp{2^{p-1}\lambda\omega^\dagger,f-T a_{m,\eps}} -\frac1p\norm{f - Ta_{m,\eps}}^p_{L^p_\pi} \leq 
    \begin{cases}
        \frac1q\norm{2^{p-1}\lambda\omega^\dagger}^q_{L^q_\pi}, \quad & p>1, \\
        \charf_{\norm{\cdot}_{L^\infty_\pi} \leq 1} (\lambda\omega^\dagger), \quad & p=1,
    \end{cases}
\end{equation*}
where $q$ is the conjugate exponent of $p$ and $\charf_{\norm{\cdot}_{L^\infty_\pi} \leq 1}$ is the indicator function of the unit ball in $L^\infty_\pi$.}\\
\indent By \cref{prop:est-norm-ameps} we have that for any $0<\delta<1$
\begin{equation*}
        \norm{a_{m,\eps}}^p_{\M} \leq \left(\frac{\eps^p}{p\lambda} + \norm{\Jminsol}_{\M} + \frac{\sqrt{2}}\delta \left(\frac{\sigma}{p\lambda} + \norm{\Jminsol}_{\M}\right) \right)^p \\
        \leq C(p) \left( \frac{\eps^{2p}}{\lambda^p} + \norm{\Jminsol}^p_{\M}\left(1 + \frac{1}{\delta^p} \right)  + \frac{\sigma^p}{\delta^p \lambda^p} \right)
\end{equation*}
with probability at least $1-\delta^2$, where $C(p)$ is a constant that depends only on $p$.\\
\indent Furthermore, from~\eqref{eq:norm-est-error} and Chebyshev's inequality, we get for any $0<\delta<1$
\begin{equation*}
    \frac{1}{m} \sum_{i=1}^m \norm{e_i}^p_{(\Y,d_*)} \leq \eps^p + \frac1\delta \frac{\sigma}{\sqrt{m}}
\end{equation*}
with probability at least $1-\delta^2$. \\
\indent Combining everything, we get that, up to multiplicative constants and higher-order terms in $\lambda$ and $\eps$, with probability at least $1-3\delta^2$
\begin{eqnarray}\label{eq:est-Bregman-long}
    \lambda D^{\subgradA^\dagger}_{\norm{\cdot}_{\M}} (a_{m,\eps}, \Jminsol) \leq  \eps^p + \frac1{\delta^p\lambda^p\sqrt{m}} + 
    \begin{cases}
        \norm{\lambda\omega^\dagger}^q_{L^q_\pi}, \quad & p>1, \\
        \charf_{\norm{\cdot}_{L^\infty_\pi} \leq 1} (\lambda\omega^\dagger), \quad & p=1
    \end{cases}
\end{eqnarray}
A division by $\lambda$ yields the claim.
}
\end{proof}

\begin{remark}
\revtwo{
The estimate in the above theorem suggests the following optimal choice of the regularisation parameter $\lambda$ and the number of samples $m$ 
\begin{equation*}
\begin{aligned}
    &\lambda \sim \eps^{p-1}, \quad m \sim \eps^{-2p^2}, \quad  &p>1, \\
    &\lambda = const  \leq \frac{1}{\norm{\omega^\dagger}}, \quad m \sim \eps^{-2}, \quad  &p=1,
\end{aligned}
\end{equation*}
under which we have
\begin{equation*}
    D^{\subgradA^\dagger}_{\norm{\cdot}_{\M}} (a_{m,\eps}, \Jminsol) = \bigO(\eps).
\end{equation*}
}
\end{remark}

\begin{remark}
\revonenew{
\cref{thm:conv-rates-Bregman} shows that under a structural assumption on the representing measure (the support separation condition, see the discussion immediately after \cref{prop:K-adj-Lip}) one can beat the curse of dimensionality not only in terms of approximating the $\Barr$ function itself, but also in terms of recovering the representing measure, albeit in a weaker metric -- the Bregman distance. As discussed in \cref{ex:Bregman-distance-Radon}, this corresponds to recovering the support of the representing measure.
}
\end{remark}

We will finish with a few remarks. 

\begin{remark}[Debiasing] As pointed out in Example~\ref{ex:Bregman-distance-Radon}, convergence in Bregman distance for the Radon norm corresponds to the convergence of the support of the regularised solution $a_{m,\eps}$ to the support of the minimum-norm solution $\Jminsol$. However, it does not say anything about the convergence of the density in regions where the support has been correctly identified. In fact, it is well known that variational regularisation introduces a bias, i.e. the value of the regularistion functional at the regularised solution $\norm{a_{m,\eps}}_{\M}$ is smaller than the value at the minimum-norm solution $\norm{\Jminsol}_{\M}$. In the mathematical imaging literature, the following procedure has been proposed to compensate for this bias~\cite{deledalle:2015-debiasing,Burger_Rasch_debiasing}. 

Let $\subgradA_{m,\eps} \in \dRadon(a_{m,\eps})$ be a subgradient of the Radon norm at the regularised solution. Often it is available as a by-product of primal-dual optimisation methods applied to~\eqref{eq:var-prob}. Consider the following problem
\begin{equation}\label{eq:debiasing}
    \min_{a \in \M(\Ball{\calL})} \frac1{pm} \sum_{i=1}^m \norm{Ta(x_i) - f(x_i) + e_i}^p_{(\Y,d_*)} \quad \text{s.t. $D^{\subgradA_{m,\eps}}_{\norm{\cdot}_{\M}}(a,a_{m,\eps}) = 0$.}
\end{equation}
Since the Bregman distance $D^{\subgradA_{m,\eps}}_{\norm{\cdot}_{\M}}(\cdot,a_{m,\eps})$ is a non-negative convex function of its first argument,~\eqref{eq:debiasing} is a convex problem because the constraint can be equivalently written as follows
\begin{equation*}
    D^{\subgradA_{m,\eps}}_{\norm{\cdot}_{\M}}(a,a_{m,\eps}) = 0 \iff \norm{a}_{\M} - \sp{a,\subgradA_{m,\eps}} \leq 0.
\end{equation*}
\end{remark}
\noindent Solutions of~\eqref{eq:debiasing} have the same support as $a_{m,\eps}$, but better fit the data. This procedure is referred to as debiasing.

%%%%%

\begin{remark}[Connection to existence theorems] 
For $p > 1$, the data term in~\eqref{eq:var-prob} is strictly convex. \revone{Since $a$ is a scalar-valued measure (all the vector-valuedness is hidden in the data term), we can use standard \emph{representer theorems}~\cite[Thm. 2 and Sect. 4.1]{unser:2020-representer}, see also~\cite{boyer:2019-representer, bredies:2020-representer},  to conclude that~\eqref{eq:var-prob} admits a sparse solution consisting of at most $m$ Dirac deltas}
\begin{equation*}
    a_{m,\eps} = \sum_{i=1}^m \alpha_i \delta_{K_i},
\end{equation*}
where $\{K_i\}_{i=1}^m \subset \Ball{\calL(\X;\Y)}$ are linear operators and $\{\alpha_i\}_{i=1}^m \subset \R$ are real numbers.  Sparsity of solutions of~\eqref{eq:var-prob} has been analysed in detail in~\cite{dias2020sparsity}. \revone{We would also like to mention the recent paper~\cite{parhi:2021}, where scalar-valued variation norm spaces with a ReLU activation are shown to be  reproducing kernel Banach spaces and a representer theorem is obtained.} 

We cannot immediately conclude, however, that this sparse solution is the same  as in Theorem~\ref{thm:direct_approximation}, for which Monte-Carlo rates in Bochner spaces $L^p_\pi$ hold. 
The connection between empirically trained neural networks (i.e. those obtained by solving~\eqref{eq:var-prob}) and the optimal ones featuring in existence-type theorems like Theorem~\ref{thm:direct_approximation} is, to the best of our knowledge, an open question.
\end{remark}

\section{Summary and conclusions}
\revone{In the first part of the paper, we studied approximation properties of two-layer neural networks as nonlinear operators acting between infinite-dimensional Banach spaces, extending existing finite-dimensional results. }
Perhaps, the most surprising result of this part is that we only obtain continuity of vector-valued $\Barr$ functions with respect to the weak-* topology on the output space $\Y$. This is unavoidable with our current techniques, but perhaps could be remedied with a different approach. Under this assumption, however, are were able to reproduce existing scalar-valued results on approximation rates. 
Our Monte-Carlo rates in $L^p_\pi$ spaces for $1 \leq p < \infty$ improve, in the scalar-valued case, on known ones from~\cite{e2020barron-representation} for all $p \neq 2$ (and coincide with them for $p=2$).

\revone{In the second part of the paper, we studied the problem of finding the best representing measure of a variation norm function from a finite number of samples}. We used variational regularisation and studied the behaviour of the regularised solution in the regime when both the number of training samples goes to infinity and the amount of noise in the samples vanishes. We obtained convergence rates in  Bregman distance, \revtwo{which hold with high probability} under the source condition, which in our case is a condition on the support of the measure that realises an $\Barr$ function. Informally, the smaller the norm of the source element, the more stable is the neural network representation of an $\Barr$ function in the sense that cancellations are avoided.

%%%%%

\section*{Acknowledgements}
The author acknowledges financial support of the EPSRC (Fellowship EP/V003615/1), the Cantab Capital Institute for the Mathematics of Information at the University of Cambridge and the National Physical Laboratory.
The author is grateful to Nik Weaver (Washington University in St. Louis) for providing the proof of Theorem~\ref{thm:lattice-op-Lip}. The author also thanks Jonas Latz (Heriot-Watt University) as well as two anonymous referees for useful remarks about the paper.

\printbibliography

%%%
\appendix

%%%%%
\section{weak-* compactness of the unit ball in the space of linear bounded operators}

\begin{prop}[{\cite[Sec.4.1, Cor. 4.8]{ryan2002book}}]\label{prop:nuclear_dual}
\revtwo{Let $E = \adj{(\predual{E})}$ and $F = \adj{(\predual{F})}$ be Banach spaces such that either $E$ or $\predual{F}$ has the approximation property. Then there exists an isometric isomorphism between $(\calN(\predual{E};\predual{F}))^*$ and $\calL(E;F)$. For any $K \in \calL(E;F)$ and $N \in \calN(\predual{E};\predual{F})$, the duality pairing can be written as follows
\begin{equation}\label{eq:nuc-lin-duality}
    \sp{N,K}  = \tr(KN^*) 
    \defeq \sum_{i=1}^\infty \sp{\eta_i, K x_i},
\end{equation}
where $\{x_i,\eta_i\}_{i \in \N} \subset E \times \predual{F}$ is any nuclear representation~\eqref{eq:nuc_rep} of $N \in \calN(\predual{E};\predual{F})$.} 
\end{prop}

\revtwo{Using the Banach-Alaoglu theorem, we obtain}
\begin{prop}\label{prop:unit-ball-op-ws-compact}
\revtwo{Suppose that $E,F$ are the duals of separable spaces $\predual{E}$ and $\predual{F}$ and that either $E$ or $\predual{F}$ has the approximation property. Then the unit ball $\Ball{\calL(E;F)}$ is weakly-* compact and metrisable.}
\end{prop}
\begin{proof}
\revtwo{Since every nuclear operator is a limit of finite-rank operators, $\calN(\predual{E};\predual{F})$ is separable and by Proposition~\ref{prop:nuclear_dual} we have that $\calL(E;F)$ is the dual of a separable Banach space. The (sequential) Banach-Alaoglu theorem implies that the unit ball $\Ball{\calL({{E}};F)}$ is weakly-* compact and metrisable.}
\end{proof}

\section{Proof of  Proposition~\ref{prop:T-compact}}\label{app:proofs}
%%%%

\begin{proof}
Consider the unit ball $\Ball{\M}$ in $\M(\Ball{\calL({{\X}};\Y)})$ and its image $T\Ball{\M}$. Take an arbitrary sequence $\{f_k\}_{k \in \N} \subset T\Ball{\M}$. It corresponds to a sequence $\{a_k\}_{k \in \N} \subset \Ball{\M}$ such that $f_k = \sp{\sigma(B\cdot),a_k}$. Since $\{a_k\}_{k \in \N}$ is bounded, it contains a weakly-* converging subsequence (which we don't relabel)
\begin{equation*}
    a_k \wsto a \quad \text{in $\Ball{\M}$}.
\end{equation*}
Fix $x \in \X$ and consider 
\begin{equation*}
    \norm{f(x)-f_k(x)}_{(\Y,d_*)} =  \sum_{i=1}^\infty 2^{-i} \abs{\sp{\eta_i,\sp{\sigma(Bx),a-a_k}}} = \sum_{i=1}^\infty 2^{-i} \abs{\sp{\sp{\eta_i,\sigma(Bx)},a-a_k}},
\end{equation*}
where $\{\eta_i\}_{i=1}^\infty \subset \predual{\Y}$ is a countable dense system in $\predual{\Y}$ as defined in Proposition~\ref{prop:rank-one-dense}. Since the functions $\sp{\eta_i,\sigma(Bx)}$ are continuous on $(\Ball{\calL(\X;\Y)},d_*)$, we have $\sp{\sp{\eta_i,\sigma(Bx)},a-a_k} \to 0$ for all $i$. \revtwo{We also have the following bound
\begin{equation}\label{eq:tmp2}
    \sum_{i=1}^\infty 2^{-i} \abs{\sp{\sp{\eta_i,\sigma(Bx)},a-a_k}} \leq \sum_{i=1}^\infty 2^{-i} \norm{a-a_k}_{\M} \norm{x} \leq C \norm{x},
\end{equation}
where we used the facts that $\norm{\eta_i}=1$, $\norm{B} \leq 1$ and $\norm{a-a_k}_{\M} \leq C$ since weakly-* convergent sequences are bounded. By the dominated convergence theorem we obtain that $\norm{f(x)-f_k(x)}_{(\Y,d_*)} \to 0$ pointwise on $\X$. Since $\pi$ has $p \geq 1$ finite moments, the function $x \mapsto \norm{x}$ is in $L^p_\pi$ and by~\eqref{eq:tmp2} and the dominated convergence theorem we get that $f_k \to f$ strongly in $L^p_\pi(\X;(\Y,d_*))$. 
Since $\{f_k\}_{k \in \N}$ was arbitrary, every sequence in $T\Ball{\M}$ has a convergent sequence and $T\Ball{\M}$ is (sequentially) compact, hence $T$ is a compact operator.}
\end{proof}

\end{document}

%% file: packages.tex
\usepackage[a4paper,bottom=3.5cm,top=2cm,left=2.54cm,right=2.54cm]{geometry}
\usepackage[english]{babel}
\usepackage[utf8]{inputenc}
\usepackage{url}
\usepackage{xcolor}
\usepackage{amsmath}
\usepackage{amssymb}
\usepackage{amsthm}
\usepackage{amsfonts}
\usepackage{bbm}
\usepackage{soul}
\usepackage{graphicx}
\usepackage{yfonts}
\usepackage[backend=bibtex,maxnames=4,url=false,style=numeric-comp,isbn=false, giveninits=true,eprint=false,sorting=none,doi=false,date=year]{biblatex}
\usepackage{alphalph}
\usepackage{courier}
\usepackage[dont-mess-around]{fnpct}
\usepackage{etoolbox}
\usepackage{comment}
\usepackage{enumerate}
\usepackage{xfrac}
\usepackage[capitalise]{cleveref}
\crefformat{equation}{(#2#1#3)}

%% file: definitions.tex
\newcommand{\defeq}{:=}
\newcommand{\norm}[1]{\left\| #1\right\|}

\newcommand{\abs}[1]{\left|#1\right|}

\newcommand{\eps}{\varepsilon}
\newcommand{\one}{\mathbbm 1}
\newcommand{\charf}{\chi}
\newcommand{\bigO}{O}

\renewcommand{\leq}{\leqslant}
\renewcommand{\geq}{\geqslant}
\renewcommand{\phi}{\varphi}

\let\sp\relax
\newcommand{\sp}[1]{\left\langle #1 \right\rangle}
\newcommand{\lsc}{lower semicontinuous}

\newcommand{\adj}[1]{#1^*}

\newcommand{\subdiff}{\partial}

\DeclareMathOperator*{\argmin}{arg\,min}

\newcommand{\reg}{\mathcal J}

\newcommand{\dJ}{\subdiff \reg}
\newcommand{\dRadon}{\subdiff \norm{\cdot}_{\M}}

\newcommand{\Jminsol}{a^\dagger}

\newcommand{\Dsymm}[1]{D_{#1}^\mathrm{symm}}
\newcommand{\D}[2]{D_{#1}^{#2}}

\DeclareMathOperator*{\supp}{supp}

\newcommand{\range}[1]{\mathcal R(#1)}

\DeclareMathOperator{\sign}{sign}

\newcommand{\predual}[1]{#1^\diamond}

\DeclareMathOperator{\tr}{Tr}
\newcommand{\mean}{{\mathbb E}}

\DeclareMathOperator{\M}{\mathcal M}

\DeclareMathOperator{\Lip}{Lip}
\newcommand{\R}{\mathbb R}
\newcommand{\Ball}[1]{{\mathbb B}_{#1}}
\newcommand{\N}{\mathbb N}

\newcommand{\weakto}{\mathrel{\rightharpoonup}}

\newcommand{\wsto}{\weakto^*}

\newtheorem{theorem}{Theorem}
\numberwithin{theorem}{section}
\newtheorem{cor}[theorem]{Corollary}
\newtheorem{lemma}[theorem]{Lemma}
\newtheorem{prop}[theorem]{Proposition}

\theoremstyle{definition}
\newtheorem{definition}[theorem]{Definition}
\newtheorem{assumption}{Assumption}
\newtheorem*{assumption*}{Assumption}
\newtheorem{remark}[theorem]{Remark}
\newtheorem*{remark*}{Remark}
\newtheorem*{definition*}{Definition}
\newtheorem{example}[theorem]{Example}

\newcommand{\calL}{\mathcal L}
\newcommand{\calN}{\mathcal N}

\newcommand{\F}{\mathcal F}

\newcommand{\X}{\mathcal X}
\newcommand{\Y}{\mathcal Y}
\newcommand{\Z}{\mathcal Z}

\let\P\relax
\newcommand{\P}{\mathcal P}

\newcommand{\C}{{\mathcal C}}

\newcommand{\Barr}{{\mathcal F_1}}
\newcommand{\subgradA}{\zeta}

\colorlet{colorone}{black}
\colorlet{colortwo}{black}
\newcommand{\revone}[1]{{\textcolor{colorone}{#1}}}
\newcommand{\revtwo}[1]{{\textcolor{colortwo}{#1}}}

\colorlet{coloronenew}{black}
\colorlet{colortwonew}{black}
\newcommand{\revonenew}[1]{{\textcolor{coloronenew}{#1}}}
\newcommand{\revtwonew}[1]{{\textcolor{colortwonew}{#1}}}

%% file: settings.tex
\graphicspath{{pic/}}
\numberwithin{equation}{section}

\usepackage{tikz, pgfplots}
\pgfplotsset{compat=newest}
\pgfplotsset{plot coordinates/math parser=false}
\usepackage[textsize=footnotesize,colorinlistoftodos]{todonotes}
\setuptodonotes{inline}
\pdfoutput=1